\documentclass{article}

\PassOptionsToPackage{numbers}{natbib}
\usepackage[final,main]{neurips_2026}

\usepackage[utf8]{inputenc}
\usepackage[T1]{fontenc}
\usepackage{amsmath,amssymb,amsthm}
\usepackage{bm}
\usepackage{algorithm}
\usepackage{algpseudocode}
\usepackage{booktabs}
\usepackage{array}
\usepackage{graphicx}
\usepackage{wrapfig}
\usepackage{subcaption}
\usepackage[table]{xcolor}
\usepackage{multirow}
\usepackage{pifont}
\usepackage{enumitem}
\usepackage{url}
\usepackage[colorlinks=true,linkcolor=blue!70!black,citecolor=green!50!black,urlcolor=magenta!70!black]{hyperref}

\makeatletter
\newif\if@inappendix
\@inappendixfalse
\let\@orig@addcontentsline\addcontentsline
\def\@apx@toc{toc}
\renewcommand{\addcontentsline}[3]{%
  \@orig@addcontentsline{#1}{#2}{#3}%
  \if@inappendix
    \def\@apx@tempa{#1}%
    \ifx\@apx@tempa\@apx@toc
      \@orig@addcontentsline{apx}{#2}{#3}%
    \fi
  \fi
}
\newcommand{\MarkAppendixInTOC}{\@inappendixtrue}
\newcommand{\AppendixOnlyTOC}{%
  \begingroup
  \par\noindent
  {\Large\bfseries Appendix Outline\par}\vspace{0.5em}
  \@starttoc{apx}%
  \par\vspace{0.5em}
  \endgroup
}
\makeatother

\newcommand{\method}{MAdam}

\newcommand{\pmv}[2]{#1\,{\scriptsize$\pm$\,#2}}
\newcommand{\best}[1]{\textbf{#1}}
\newcommand{\up}{$\uparrow$}
\newcommand{\down}{$\downarrow$}
\definecolor{highlight}{RGB}{210,230,255}
\definecolor{tableblue}{RGB}{135,180,235}
\definecolor{tablelightblue}{RGB}{210,230,255}
\definecolor{02lightgray}{RGB}{180,180,180}
\newcommand{\rlightgray}{\rowcolor{02lightgray!25}}
\newcolumntype{G}{>{\columncolor{02lightgray!25}[\tabcolsep][0pt]}c}

\theoremstyle{definition}
\newtheorem{definition}{Definition}
\newtheorem{proposition}{Proposition}
\newtheorem{lemma}{Lemma}

\theoremstyle{remark}
\newtheorem{remark}{Remark}

\title{\method{}: Metric-Aware Multi-Objective Adam}
\author{
Fengbei Liu$^{1}$ \quad
Rachit Saluja$^{1}$ \quad
Sunwoo Kwak$^{1}$ \quad
Ruibo Wang$^{3}$ \quad
Ruining Deng$^{2}$ \\
\textbf{Heejong Kim}$^{2}$ \quad
\textbf{Johannes C. Paetzold}$^{1,2}$ \quad
\textbf{Mert R. Sabuncu}$^{1,2}$\\[0.5em]
$^{1}$Cornell Tech \quad
$^{2}$Weill Cornell Medicine \quad
$^{3}$TU Delft
}

\begin{document}

\maketitle

\begin{abstract}
Multi-objective optimization (MOO) underlies many machine learning problems, yet MOO solvers across the loss-balancing, gradient-balancing, and Pareto-based families almost universally hand their reconciled directions to Adam~\citep{kingma2015adam}. We show this coupling introduces two systematic gaps between the solver's intent and the optimizer's execution. The first is a \emph{weighting mismatch}: Adam's second-moment denominator entangles the time-varying preference vector with gradient statistics, marginalizing the preference into a history average and collapsing distinct Pareto trade-offs toward a near-uniform mixture. The second is a \emph{geometric mismatch}: Adam's adaptive metric distorts the Euclidean geometry MOO solvers assume, turning aligned objectives into apparent conflicts.
To resolve both jointly, we introduce \textbf{\method{}} (Metric-Aware Multi-Objective Adam), a drop-in wrapper that leaves both solver and optimizer unchanged. \method{} preconditions the reconciled direction by the preference-conditioned curvature of the scalarized objective; on this whitened input, Adam's second moment collapses to identity, so the realized update is governed by the preference-conditioned metric.
Across multi-task learning, Pareto-front recovery, physics-informed neural networks, and medical imaging, \method{} improves over Adam in aggregate for every solver family.
Our code is available at \url{https://github.com/lfb-1/MAdam}.
\end{abstract}

\section{Introduction}\label{sec:intro}

Many practical learning problems require a shared model to satisfy multiple objectives simultaneously. Such settings are naturally cast as multi-objective optimization (MOO). Three MOO solver paradigms have emerged~\citep{chen2025gradient}: \emph{loss-balancing} methods adapt objective weights to rebalance gradient magnitudes~\citep{sener2018multi,navon2022multi,liu2023famo,lin2024smooth}; \emph{gradient-balancing} methods modify objective gradients to reduce inter-objective conflict~\citep{yu2020gradient,liu2021conflict,liu2021imtl}; and \emph{Pareto-based} methods steer the update toward a Pareto front target region~\citep{lin2019pareto,ruchte2021scalable,dimitriadis2025palora,chen2024ferero}. Most gradient based MOO solvers aim to produce a single reconciled direction.

However, gradient-based optimization relies on two steps: \emph{constructing} the descent direction and \emph{executing} the update with a suitable metric, which is a role all three MOO solver families above unanimously delegate to Adam~\citep{kingma2015adam}. Adam implicitly imposes its own adaptive diagonal metric via the second-moment denominator, which under stationarity approximates a diagonal empirical Fisher information matrix~\citep{kunstner2019limitations,martens2020new}. We show this default pairing introduces two systematic mismatches between what the solver intends and what the optimizer actually executes.

Firstly, in loss-balancing and Pareto-based methods, the solver's reconciled direction depends on a time-varying preference vector $\boldsymbol{\lambda}^{(t)}$: loss-balancing methods adapt these preferences to rebalance gradient magnitudes, while Pareto-based methods adjust them to steer toward a target Pareto region. However, Adam's running second moment entangles these preferences with the objective gradient statistics into a single surrogate, weakening the intended preference strategy and collapsing distinct Pareto trade-offs into a near-uniform average (Prop.~\ref{prop:adam-weighting-mismatch}).

Secondly, MOO solvers typically assume a raw Euclidean geometry: loss-balancing methods explicitly weight the objectives through linear scalarization. Gradient-balancing methods measure gradient similarities and project to reduce conflict in Euclidean space. Pareto-based methods steer a Pareto target region through linear scalarization of user preferences. However, Adam is not metric-neutral: its second-moment denominator imposes a diagonal RMS metric on every update, conflicting with the Euclidean metric assumed by MOO solvers. This geometric mismatch distorts the solver's reconciled direction, turning aligned objectives into apparent conflicts and vice versa (Prop.~\ref{prop:adam-geometry}).

These observations raise two questions: \textbf{Q1. Which metric should govern MOO updates?} \textbf{Q2. How can it be realized within the existing solver--Adam pipeline? }

We propose \method{} (Metric-aware Multi-objective Adam), which addresses both questions with a single mechanism. For Q1, we derive the \emph{preference-conditioned} diagonal Fisher information matrix of the scalarized objective at the active preference, given by the second moment of the scalarized gradient and admitting a decomposition into within- and cross-objective Fisher blocks. For Q2, we precondition the solver's reconciled direction by this diagonal metric prior to the Adam update, modifying neither the solver nor the optimizer. Applying the correction at Adam's input keeps its first- and second-moment EMAs mutually consistent, so the realized parameter update is governed by the preference-conditioned metric. Our contributions are threefold
\begin{itemize}
    \item \textbf{Diagnosis of the solver--Adam mismatch.} We prove two failure modes of solver--Adam coupling in MOO: a \emph{weighting mismatch} that marginalizes the preference (Prop.~\ref{prop:adam-weighting-mismatch}) and a \emph{geometric mismatch} that distorts the MOO solvers' Euclidean geometry (Prop.~\ref{prop:adam-geometry}).
    \item \textbf{Metric-aware gradient preconditioning.} We derive the preference-conditioned diagonal Fisher of the scalarized objective and apply it as a plug-and-play preconditioner on the reconciled direction, resolving both mismatches with a single mechanism.
\end{itemize}
We validate MAdam across diverse MOO solvers and observe consistent improvements over Adam on multi-task learning, Pareto-front recovery, PINN benchmarks, and medical imaging applications.

\section{Related Work}\label{sec:related}

\paragraph{Multi-Task and Pareto MTL.}
Multi-task learning reconciles competing objectives by rescaling per-objective losses, manipulating task gradients, or targeting a preference-conditioned point on the Pareto front.
Representative methods range from fixed or learned scalarizations (linear scalarization, uncertainty weighting~\citep{kendall2018uncertainty}) through dynamically adapted weights (DWA~\citep{liu2019dwa}, IMTL~\citep{liu2021imtl}, CAGrad~\citep{liu2021conflict}) to preference-conditioned models that recover the full Pareto front from a single network (PaMaL~\citep{dimitriadis2023pamal}, PaLoRA~\citep{dimitriadis2025palora}).
Across this family the constructed direction is ultimately consumed by Adam, whose diagonal second-moment EMA entangles the time-varying preference $\boldsymbol{\lambda}^{(t)}$ with per-objective gradient statistics and imposes a task-agnostic RMS metric on each step.

\paragraph{PINN and MIA Loss Balancing.}
Physics-informed neural networks (PINNs) and medical-image analysis (MIA) exhibit a closely related imbalance: residual, boundary, and reconstruction losses span widely different scales, so a single objective dominates absent explicit rebalancing.
The PINNacle benchmark~\citep{hao2024pinnacle} consolidates remedies along three axes, loss reweighting (LRA~\citep{wang2021understanding}, NTK~\citep{wang2022ntk}), residual-based sampling (RAR~\citep{wu2023comprehensive}), and adaptive activations (LAAF/GAAF~\citep{jagtap2020adaptive}); in MIA, vanilla Adam~\citep{paszke2019pytorch} and random loss weighting~\citep{lin2022reasonable} remain the prevailing choices.
Such first-order remedies adjust the descent direction yet leave the Adam optimizer's step metric untouched.

\paragraph{Curvature-Aware Optimizers.}
A separate line of curvature-aware optimizers, spanning natural gradient~\citep{amari1998natural,martens2020new}, K-FAC~\citep{martens2015optimizing}, Shampoo~\citep{gupta2018shampoo}, AdaHessian~\citep{yao2021adahessian}, Sophia~\citep{liu2023sophia}, and FAdam~\citep{hwang2024fadam}, targets a single scalar objective and supplants the first-order optimizer wholesale.
None preconditions the reconciled direction of a multi-objective solver under a time-varying preference.
\method{} occupies this gap: it derives a preference-conditioned diagonal Fisher of the scalarized objective and inserts it into the solver-then-Adam pipeline as a wrapper, leaving both the MOO solver and the optimizer intact.
A comprehensive listing of MTL, PINN, MIA, and curvature-aware methods is in App.~\ref{app:related-work-comprehensive}.

\section{Method}\label{sec:method}
Sec.~\ref{sec:background} defines preliminaries about MOO solver and Adam. Sec.~\ref{sec:mismatch} identifies two problems Adam has on reconciled MOO directions: a weight mismatch and a geometric mismatch. Sec.~\ref{sec:q1} derives the correct preference-conditioned curvature (Q1). Sec.~\ref{sec:q2} applies it by preconditioning the solver's output before passing it to Adam (Q2). Sec.~\ref{sec:practical} describes the practical implementation, and Algorithm~\ref{alg:madam} states the full procedure.
\subsection{Preliminaries}\label{sec:background}

\paragraph{MOO Solver Setup.}
We consider $C$ objectives with per-objective losses $\{\ell_i(\boldsymbol{\theta})\}_{i=1}^C$ under shared model parameters $\boldsymbol{\theta} \in \mathbb{R}^d$.
In many scenarios, learning can be formulated as:
$\boldsymbol{\theta}^*(\boldsymbol{\lambda}^{(t)}) = \arg \min_{\boldsymbol{\theta}} \sum_{i=1}^C \lambda_i^{(t)} \ell_i(\boldsymbol{\theta}),$
where $\boldsymbol{\lambda}^{(t)} \in \Delta^{C-1}$ is a per-step preference vector on the probability simplex.
At iteration $t$, let $\mathbf{g}_i^{(t)} := \nabla_{\boldsymbol{\theta}}\ell_i(\boldsymbol{\theta}^{(t)})$ denote the per-objective gradient.

\begin{definition}[MOO Solver]\label{def:moo-solver}
Assume a \emph{MOO solver} produces a reconciled update direction through linear scalarization:
\begin{equation}\label{eq:moo-solver}
    \mathbf{d}^{(t)} \;=\; \sum_{i=1}^{C} \lambda_i^{(t)}\,\mathbf{g}_i^{(t)}
    \;=\; \nabla_{\boldsymbol{\theta}}\,\ell_{\boldsymbol{\lambda}^{(t)}}\!\bigl(\boldsymbol{\theta}^{(t)}\bigr),
    \qquad
    \ell_{\boldsymbol{\lambda}^{(t)}} \;:=\; \sum_{i=1}^{C} \lambda_i^{(t)}\,\ell_i.
\end{equation}
The definition requires only a per-step linear combination of per-objective gradients whose coefficients $\boldsymbol{\lambda}^{(t)}$ are detached from $\boldsymbol{\theta}$ in the backward pass; the second equality then follows from linearity of the gradient. It does not require a global scalarized objective or simplex-valued coefficients: $\boldsymbol{\lambda}^{(t)}$ may be signed or unnormalized, as in IMTL, CAGrad, and Nash-MTL. This form unifies three families of MOO solvers, distinguished by how $\boldsymbol{\lambda}^{(t)}$ is specified: (i) \emph{loss balancing}, where $\boldsymbol{\lambda}^{(t)}\in\Delta^{C-1}$ is computed from the latest per-objective losses; (ii) \emph{gradient balancing}, where $\boldsymbol{\lambda}^{(t)}\in\mathbb{R}^{C}$ is obtained from the geometry of the per-objective gradients (e.g., MGDA, PCGrad, CAGrad); and (iii) \emph{preference-based Pareto}, where $\boldsymbol{\lambda}^{(t)}$ is an externally supplied preference vector that is swept to steer the iterate along the Pareto front. In all three cases, the reconciled direction $\mathbf{d}^{(t)}$ is the \emph{Euclidean gradient} of the scalarized loss $\ell_{\boldsymbol{\lambda}^{(t)}}$ on $\mathbb{R}^d$. Specific solvers in each family are detailed in App.~\ref{app:moo-solvers}.
\end{definition}

\paragraph{Adam as a Diagonal RMS Preconditioner.}
Adam~\citep{kingma2015adam} maintains exponential moving averages (EMAs) of the first and second moments of the supplied direction $\mathbf{d}^{(t)}$:
\begin{equation}\label{eq:adam_first_second_moment}
    \mathbf{m}^{(t)} = \beta_1 \mathbf{m}^{(t-1)} + (1-\beta_1)\mathbf{d}^{(t)},
    \qquad
    \mathbf{v}^{(t)} = \beta_2 \mathbf{v}^{(t-1)}
        + (1-\beta_2)\,\mathbf{d}^{(t)} \odot \mathbf{d}^{(t)},
\end{equation}
where $\odot$ denotes element-wise multiplication, and applies the bias-corrected $\hat{\mathbf{m}}^{(t)}$ and $\hat{\mathbf{v}}^{(t)}$:
\begin{equation}\label{eq:adam_update}
    \boldsymbol{\theta}^{(t+1)}
    = \boldsymbol{\theta}^{(t)}
    - \eta\,\frac{\hat{\mathbf{m}}^{(t)}}{\sqrt{\hat{\mathbf{v}}^{(t)}} + \varepsilon},
\end{equation}
where $\varepsilon$ is a small positive constant.
Eq.~\eqref{eq:adam_update} can equivalently be written in preconditioned form as
\begin{equation}\label{eq:adam-metric}
    \boldsymbol{\theta}^{(t+1)} = \boldsymbol{\theta}^{(t)} - \eta\,\bigl(\mathbf{M}_{\mathrm{Adam}}^{(t)}\bigr)^{-1}\hat{\mathbf{m}}^{(t)},
    \qquad
    \mathbf{M}_{\mathrm{Adam}}^{(t)} := \operatorname{Diag}\!\bigl(\sqrt{\hat{\mathbf{v}}^{(t)}} + \varepsilon\bigr),
\end{equation}
with $\operatorname{Diag}(\cdot)$ the diagonal matrix formed from its vector argument. In a stationary regime, $\hat{\mathbf{v}}^{(t)} \approx \mathbb{E}[\mathbf{d}^{(t)}\odot\mathbf{d}^{(t)}]$ recovers the diagonal of the empirical Fisher information matrix of the update direction~\citep{kunstner2019limitations,martens2020new}, so that $\bigl(\mathbf{M}_{\mathrm{Adam}}^{(t)}\bigr)^{2} \approx \operatorname{Diag}\!\bigl(\mathbf{F}^{(t)}\bigr)$. Eq.~\eqref{eq:adam-metric} therefore realizes approximate natural gradient descent whose metric is a diagonal empirical Fisher information matrix~\citep{amari1998natural,martens2020new,hwang2024fadam};

\subsection{Adam and Its Problems in MOO}\label{sec:mismatch}

\begin{figure}[!t]
    \centering
    \includegraphics[width=\linewidth]{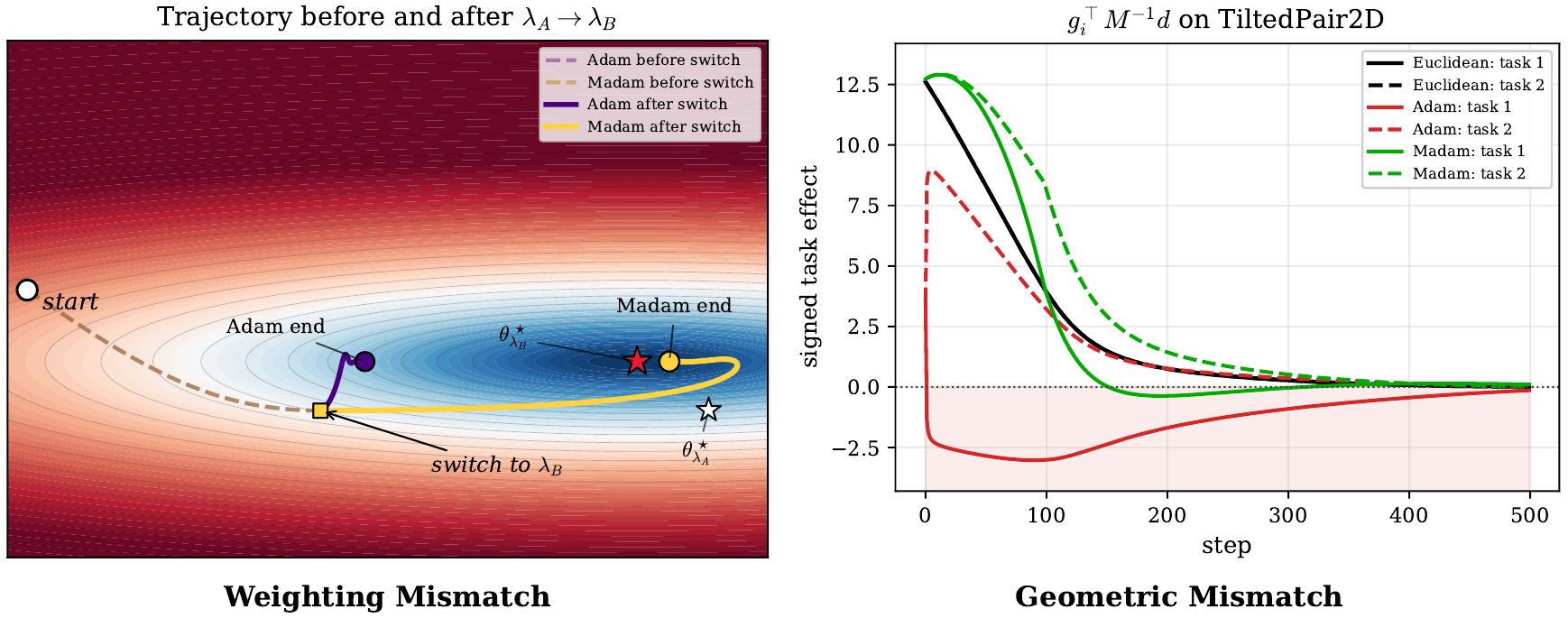}
    \caption{\textbf{Two sources of solver--Adam mismatch on a 2D tilted-quadratic MOO problem.}
    \emph{Left (weighting mismatch):} decision-space trajectories across a single preference switch $\boldsymbol{\lambda}_A\!\to\!\boldsymbol{\lambda}_B$; contours show the post-switch loss $\ell_{\boldsymbol{\lambda}_B}$ (darker is lower) and the star marks its minimum $\boldsymbol{\theta}^\star_{\boldsymbol{\lambda}_B}$. Adam's second-moment EMA still encodes the pre-switch $\boldsymbol{\lambda}_A$, so its iterate stalls near $\boldsymbol{\theta}^\star_{\boldsymbol{\lambda}_A}$, while \method{} re-conditions to the active preference and reaches $\boldsymbol{\theta}^\star_{\boldsymbol{\lambda}_B}$ (Prop.~\ref{prop:adam-weighting-mismatch}; multi-phase version in App.~\ref{app:multiphase-weighting}).
    \emph{Right (geometric mismatch):} per-objective signed effect $\mathbf{g}_i^{\!\top}\mathbf{M}^{-1}\mathbf{d}$ on the same problem. Under the Euclidean metric ($\mathbf{M}=\mathbf{I}$) it is non-negative for every task, so any negative value is a sign flip against the solver's intent. Adam's preconditioner drives task~1 strongly negative (shaded), so the update increases $\ell_1$ although the solver treats task~1 as aligned with $\mathbf{d}$; \method{} tracks the Euclidean reference and largely removes the flip (Prop.~\ref{prop:adam-geometry}; mechanism in App.~\ref{app:sign-flip}).}
    \label{fig:mismatch-side-by-side}
\end{figure}

\paragraph{Source 1: Weighting Mismatch.}

Adam preconditions its input as though the preference were stationary. When the MOO solver supplies a non-stationary $\boldsymbol{\lambda}^{(t)}$, as in loss-balancing and Pareto-based methods, Adam's second-moment EMA conflates preference weights with gradient statistics and absorbs the current preference into a history-averaged quantity rather than applying it at each step. Fig.~\ref{fig:mismatch-side-by-side} (left) illustrates the consequence: after a preference switch $\boldsymbol{\lambda}_A\!\to\!\boldsymbol{\lambda}_B$, Adam's iterate remains biased toward the previous optimum $\boldsymbol{\theta}^\star_{\boldsymbol{\lambda}_A}$, whereas \method{} converges to $\boldsymbol{\theta}^\star_{\boldsymbol{\lambda}_B}$.

\begin{proposition}[Adam Tracks Weight-Marginalized Curvature]\label{prop:adam-weighting-mismatch}
Assume $\boldsymbol{\lambda}^{(\tau)}$ is independent of the per-objective gradient cross-moments $\mathbf{g}_i^{(\tau)}\odot\mathbf{g}_j^{(\tau)}$ within the EMA window. This assumption holds for loss-balancing solvers, where $\boldsymbol{\lambda}^{(\tau)}$ is determined by scalar loss statistics that vary on a slower timescale than gradient noise, and for preference-based Pareto methods, where $\boldsymbol{\lambda}^{(\tau)}$ is supplied externally. Under this independence and the stationarity approximation in Eq.~\eqref{eq:adam-metric}, Adam's second-moment EMA expands as
\begin{equation}\label{eq:ema-expanded}
    \hat{\mathbf{v}}^{(t)}
    \;\approx\; \mathbb{E}\!\bigl[\mathbf{d}^{(t)}\odot\mathbf{d}^{(t)}\bigr]
    \;=\; \sum_{i,j=1}^{C}
        \mathbb{E}\!\bigl[\lambda_i\lambda_j\,\mathbf{g}_i\odot\mathbf{g}_j\bigr]
    \;=\; \sum_{i,j=1}^{C} \mathbb{E}[\lambda_i\lambda_j]\,\mathbf{F}_{ij}.
\end{equation}
\end{proposition}

The middle expression in~\eqref{eq:ema-expanded} makes the entanglement explicit: each summand is a single expectation of $\lambda_i\lambda_j\,\mathbf{g}_i\odot\mathbf{g}_j$ taken jointly. Adam cannot factor this expectation into separate preference and gradient EMAs, so the current weights $\lambda_i^{(t)}\lambda_j^{(t)}$ are subsumed into the history-averaged statistic $\mathbb{E}[\lambda_i\lambda_j]$, which we refer to as the \emph{weight marginalization}. In loss-balancing solvers~\citep{sener2018multi,navon2022multi,liu2023famo,lin2024smooth}, each $\boldsymbol{\lambda}^{(t)}$ encodes a deliberate re-weighting that Adam's EMA absorbs into a history-averaged $\mathbb{E}[\lambda_i\lambda_j]$, dampening the strategy in the actual update. In preference-based Pareto methods~\citep{lin2019pareto,ruchte2021scalable,dimitriadis2025palora,chen2024ferero} that sweep $\boldsymbol{\lambda}^{(t)}$ within a single run to approximate Pareto front, the marginalization collapses the executed geometry toward a near-uniform mixture, lowering the diversity of Pareto trade-offs.

\paragraph{Source 2: Geometric Mismatch.}

We now turn to general linear scalarization. The linear scalarization in Definition~\ref{def:moo-solver} treats $\mathbf{d}^{(t)}$ as a direction under the Euclidean metric. In practice, however, when Adam is used as the optimizer, the step is executed under Adam's diagonal RMS metric $\mathbf{M}_{\mathrm{Adam}}^{(t)}$~\eqref{eq:adam-metric}, which differs from the Euclidean identity. We now quantify how this discrepancy distorts the per-objective loss change realized by each step.

\begin{proposition}[Per-objective Loss Change with Adam]\label{prop:adam-geometry}
Let $\mathbf{d}^{(t)}$ be the reconciled direction supplied to Adam and $\hat{\mathbf{m}}^{(t)}$ Adam's bias-corrected first-moment EMA of its input. Under the first-order Taylor approximation $\Delta\ell_i^{(t)} \approx \mathbf{g}_i^{(t)\top}\Delta\boldsymbol{\theta}^{(t)}$, the per-objective loss change realized by Adam is
\begin{equation}\label{eq:actual-effect}
    \Delta\ell_i^{(t)}
    \approx -\eta\,\mathbf{g}_i^{(t)\top}\bigl(\mathbf{M}_{\mathrm{Adam}}^{(t)}\bigr)^{-1}\hat{\mathbf{m}}^{(t)}.
\end{equation}
Under the stationarity approximation $\hat{\mathbf{m}}^{(t)}\approx\mathbf{d}^{(t)}$, this reduces to $-\eta\,\mathbf{g}_i^{(t)\top}(\mathbf{M}_{\mathrm{Adam}}^{(t)})^{-1}\mathbf{d}^{(t)}$.
\end{proposition}
\begin{proof}
See App.~\ref{app:proof-prop-adam-geometry}.
\end{proof}
Eq.~\eqref{eq:actual-effect} expresses the per-objective loss change as a bilinear form in the per-objective gradient $\mathbf{g}_i^{(t)}$ and the first-moment direction $\hat{\mathbf{m}}^{(t)}$ (which coincides with the reconciled $\mathbf{d}^{(t)}$ under stationarity), preconditioned by $(\mathbf{M}_{\mathrm{Adam}}^{(t)})^{-1}$. For $C=1$, $\mathbf{g}_i^{(t)}\!=\!\mathbf{d}^{(t)}$ and the form collapses to the quadratic $-\eta\,\mathbf{d}^{(t)\top}(\mathbf{M}_{\mathrm{Adam}}^{(t)})^{-1}\mathbf{d}^{(t)}\!\leq\!0$, guaranteeing descent for any positive-definite metric. For $C>1$, however, $\mathbf{g}_i^{(t)}\!\neq\!\mathbf{d}^{(t)}$ in general and the bilinear form is indefinite: $(\mathbf{M}_{\mathrm{Adam}}^{(t)})^{-1}$ can flip the sign of $\mathbf{g}_i^{(t)\top}(\mathbf{M}_{\mathrm{Adam}}^{(t)})^{-1}\mathbf{d}^{(t)}$ relative to $\mathbf{g}_i^{(t)\top}\mathbf{d}^{(t)}$, converting an aligned objective into an apparent conflict (App.~\ref{app:sign-flip-geometric}). Fig.~\ref{fig:mismatch-side-by-side} (right) instantiates this regime on a synthetic two-task problem: Adam's preconditioner drives task~1's signed effect $\mathbf{g}_1^{(t)\top}(\mathbf{M}_{\mathrm{Adam}}^{(t)})^{-1}\mathbf{d}^{(t)}$ strongly negative, producing a large positive $\Delta\ell_1^{(t)}$: the update actively increases the loss the solver intended to decrease.

\subsection{Q1: The Correct Curvature for MOO}\label{sec:q1}

As we have discussed above, 
Adam's EMA averages the entangled product of $\boldsymbol{\lambda}^{(t)}$ with the per-objective gradients inside $\mathbf{d}^{(t)}$, rather than the two factors separately. 
Our proposed remedy is to accumulate gradient cross-moments independently and combine them with the current preference at each step, instead of letting a joint EMA absorb the mixing. The resulting curvature estimate is
\begin{equation}\label{eq:scalarized-curvature}
    \begin{split}
     \mathbf{C}_{\boldsymbol{\lambda}^{(t)}}
     = \mathbb{E}\!\left[\mathbf{g}_{\boldsymbol{\lambda}^{(t)}} \odot \mathbf{g}_{\boldsymbol{\lambda}^{(t)}}\right]
    = \sum_{i,j=1}^{C}\lambda_i^{(t)}\lambda_j^{(t)}\,\mathbf{F}_{ij} = \sum_{i=1}^{C} \bigl(\lambda_i^{(t)}\bigr)^2\, \mathbf{F}_{ii} \;+\; \sum_{i \neq j} \lambda_i^{(t)}\lambda_j^{(t)}\, \mathbf{F}_{ij},
    \end{split}
\end{equation}
the diagonal second moment of the scalarized gradient at the current preference $\boldsymbol{\lambda}^{(t)}$. The expectation parallels Adam's $\mathbb{E}[\mathbf{d}^{(t)}\odot\mathbf{d}^{(t)}]$ in~\eqref{eq:ema-expanded} but holds $\boldsymbol{\lambda}^{(t)}$ fixed at the current step, so the deterministic weights factor out and only the gradient cross-moments remain inside the expectation. The second line separates the diagonal terms $\mathbf{F}_{ii}$ (per-objective curvature) from the off-diagonal terms $\mathbf{F}_{ij}$ (cross-objective gradient correlation).

Two properties confirm $\mathbf{C}_{\boldsymbol{\lambda}^{(t)}}$ as the appropriate metric for the current trade-off. \textbf{(i)~Current-preference alignment.} The weights $\lambda_i^{(t)}\lambda_j^{(t)}$ appear outside the expectation and $\mathbf{F}_{ij}$ inside, so the current preference directly controls the contribution of each cross-moment rather than being smoothed into a window average. \textbf{(ii)~Cross-objective structure.} The off-diagonal terms $\mathbf{F}_{ij}$ ($i\neq j$) capture coordinate-wise gradient correlation between tasks: positive entries indicate coordinates where tasks cooperate, negative entries indicate conflict. Adam collapses both diagonal and off-diagonal contributions into a single per-coordinate magnitude; $\mathbf{C}_{\boldsymbol{\lambda}^{(t)}}$ keeps them separable and adjustable as $\boldsymbol{\lambda}^{(t)}$ changes. We define the corresponding preference-conditioned metric as
\begin{equation}\label{eq:target-metric}
    \mathbf{M}_{\boldsymbol{\lambda}^{(t)}}
    := \operatorname{Diag}\!\bigl(\sqrt{\mathbf{C}_{\boldsymbol{\lambda}^{(t)}}} + \varepsilon\bigr).
\end{equation}

\subsection{Q2: Metric Preconditioning of the Solver Output}\label{sec:q2}

Q1 supplies the target metric $\mathbf{M}_{\boldsymbol{\lambda}^{(t)}}$. Eq.~\eqref{eq:actual-effect} offers three intervention points: (i)~modify $\mathbf{g}_i^{(t)}$, but each MOO solver handles per-objective gradients differently, requiring a method-specific modification for each solver; (ii)~replace $\mathbf{M}_{\mathrm{Adam}}^{(t)}$ with $\mathbf{M}_{\boldsymbol{\lambda}^{(t)}}$ in Adam's denominator, but this adjusts only the second-moment EMA while the first-moment EMA remains built from unscaled inputs, leaving the numerator and denominator inconsistent; (iii)~precondition the reconciled direction $\mathbf{d}^{(t)}$, the single output common to all MOO solvers, which propagates the preference correction into both EMAs simultaneously without modifying the solver or the optimizer. 

We adopt approach~(iii) and supply Adam with the preconditioned direction $\mathbf{M}_{\boldsymbol{\lambda}^{(t)}}^{-1}\mathbf{d}^{(t)}$ in place of $\mathbf{d}^{(t)}$. By construction $\mathbf{M}_{\boldsymbol{\lambda}}$ whitens the scalarized gradient, so Adam's second-moment EMA on this input collapses to identity:
\begin{equation}\label{eq:adam-whitened}
    \hat{\mathbf{v}}^{(t)}
    \;\approx\; \mathbb{E}\!\bigl[\mathbf{M}_{\boldsymbol{\lambda}}^{-1}\mathbf{d}^{(t)} \odot \mathbf{M}_{\boldsymbol{\lambda}}^{-1}\mathbf{d}^{(t)}\bigr]
    \;=\; \mathbf{M}_{\boldsymbol{\lambda}}^{-2}\,\mathbf{C}_{\boldsymbol{\lambda}}
    \;\approx\; \mathbf{1},
    \qquad
    \mathbf{M}_{\mathrm{Adam}}^{(t)} \;\approx\; \mathbf{I},
\end{equation}
making Adam's denominator a pass-through. With the preconditioned input $\mathbf{M}_{\boldsymbol{\lambda}^{(t)}}^{-1}\mathbf{d}^{(t)}$ fed to Adam, stationarity gives $\hat{\mathbf{m}}^{(t)}\approx\mathbf{M}_{\boldsymbol{\lambda}^{(t)}}^{-1}\mathbf{d}^{(t)}$ and~\eqref{eq:adam-whitened} gives $\mathbf{M}_{\mathrm{Adam}}^{(t)}\approx\mathbf{I}$; substituting both into~\eqref{eq:actual-effect}:
\begin{equation}\label{eq:cancellation}
    \Delta\ell_i^{(t)}
    \;\approx\; -\eta\,\mathbf{g}_i^{(t)\top}\bigl(\mathbf{M}_{\mathrm{Adam}}^{(t)}\bigr)^{-1}\mathbf{M}_{\boldsymbol{\lambda}^{(t)}}^{-1}\mathbf{d}^{(t)}
    \;\approx\; -\eta\,\mathbf{g}_i^{(t)\top}\mathbf{M}_{\boldsymbol{\lambda}^{(t)}}^{-1}\mathbf{d}^{(t)}.
\end{equation}
The reduction of $\mathbf{M}_{\mathrm{Adam}}^{(t)}$ to the identity does not recover the Euclidean metric: it replaces an uncontrolled, history-marginalized preconditioner with the preference-conditioned $\mathbf{M}_{\boldsymbol{\lambda}^{(t)}}^{-1}$, which is precisely the curvature established in Q1. The realized parameter update is therefore
\begin{equation}\label{eq:ideal-update}
    \boldsymbol{\theta}^{(t+1)}
    \;=\; \boldsymbol{\theta}^{(t)}
    - \eta\,\mathbf{M}_{\boldsymbol{\lambda}^{(t)}}^{-1}\mathbf{d}^{(t)}.
\end{equation}

\subsection{Practical Estimation}\label{sec:practical}

\paragraph{Cross-Objective Fisher Estimation.}
The metric $\mathbf{M}_{\boldsymbol{\lambda}}$ requires the cross-objective Fisher interactions $\mathbf{F}_{ij}$, which are not directly available. We estimate them online with EMAs that share the same decay rate $\beta_2$ as Adam, then assemble and rampup-blend the resulting curvature into the preconditioner:
\begin{equation}\label{eq:fisher-ema}
    \widehat{\mathbf{F}}_{ij}^{(t)}
    = \beta_2\,\widehat{\mathbf{F}}_{ij}^{(t-1)}
    + (1-\beta_2)\,\mathbf{g}_i^{(t)} \odot \mathbf{g}_j^{(t)},
    \qquad 1 \le i \le j \le C,
\end{equation}
\begin{equation}\label{eq:practical-metric}
    \widehat{\mathbf{C}}_{\boldsymbol{\lambda}}^{(t)}
    := \sum_{i,j=1}^{C} \lambda_i^{(t)}\lambda_j^{(t)}\,\widehat{\mathbf{F}}_{ij}^{(t)},
    \qquad
    \widehat{\mathbf{M}}_{\boldsymbol{\lambda}}^{(t)}
    := \operatorname{Diag}\!\Bigl(\alpha^{(t)}\sqrt{\widehat{\mathbf{C}}_{\boldsymbol{\lambda}}^{(t)} + \varepsilon} + (1-\alpha^{(t)})\Bigr),
\end{equation}
where the rampup coefficient $\alpha^{(t)}\in[0,1]$ rises monotonically from $0$ to $1$ over a fixed warmup window, interpolating the denominator between identity (Adam on $\mathbf{d}^{(t)}$) and the full \method{} preconditioner. The rampup defers reliance on $\widehat{\mathbf{F}}_{ij}^{(t)}$ until its EMA has stabilized, since early estimates are dominated by noise. Replacing Eq.~\eqref{eq:ideal-update} $\mathbf{M}_{\boldsymbol{\lambda}}$ with $\widehat{\mathbf{M}}_{\boldsymbol{\lambda}}^{(t)}$ yields the practical algorithm.

\paragraph{Stochastic Pair Sampling.}
Maintaining the $C(C+1)/2$ symmetric EMA entries~\eqref{eq:fisher-ema} requires the per-objective gradient $\mathbf{g}_i^{(t)}$ for every task, incurring $O(C)$ backward passes per step. We reduce this to a constant by drawing a pair $(i,j)\sim\mathrm{Uniform}([C]\times[C])$ at each iteration, computing only $\mathbf{g}_i^{(t)}$ and $\mathbf{g}_j^{(t)}$, and updating the three EMAs that depend on the pair, namely $\widehat{\mathbf{F}}_{ii}^{(t)}$, $\widehat{\mathbf{F}}_{jj}^{(t)}$, and $\widehat{\mathbf{F}}_{ij}^{(t)}$ (with $\widehat{\mathbf{F}}_{ji}$ identified by symmetry); all other entries are retained from step $t{-}1$. The MOO solver is unchanged and operates on the full task set, computing $\mathbf{d}^{(t)}$ as in Definition~\ref{def:moo-solver}. This reduces to $O(1)$ backward passes, twice per step, regardless of $C$. 

\begin{algorithm}[t]
\caption{\method{}: Metric-Aware Multi-Objective Adam}
\label{alg:madam}
\begin{algorithmic}[1]
    \Require MOO solver, $C$ tasks, initial $\boldsymbol{\theta}^{(0)}$, Adam hyperparameters $\eta,\beta_1,\beta_2,\varepsilon$
    \State Initialize Adam state; $\widehat{\mathbf{F}}_{ij}\gets\mathbf{0}$ for $1\le i\le j\le C$
    \For{$t = 1, 2, \dots$}
        \State Receive preference $\boldsymbol{\lambda}^{(t)}\in\Delta^{C-1}$
        \State Sample pair $(i,j)\sim\mathrm{Uniform}([C]\times[C])$ and compute $\mathbf{g}_i^{(t)},\mathbf{g}_j^{(t)}$
        \State Run MOO solver on the full task set to obtain reconciled direction $\mathbf{d}^{(t)}$
        \State \textcolor{blue!80!black}{Update $\widehat{\mathbf{F}}_{ii}^{(t)}, \widehat{\mathbf{F}}_{jj}^{(t)}, \widehat{\mathbf{F}}_{ij}^{(t)}$ via~\eqref{eq:fisher-ema} (with $\widehat{\mathbf{F}}_{ji}^{(t)}=\widehat{\mathbf{F}}_{ij}^{(t)}$)} \hfill \textcolor{blue!80!black}{\textit{// step 1: Fisher EMAs}}
        \State \textcolor{blue!80!black}{Form $\widehat{\mathbf{C}}^{(t)}, \widehat{\mathbf{M}}^{(t)}$ via~\eqref{eq:practical-metric}} \hfill \textcolor{blue!80!black}{\textit{// step 2: assemble metric}}
        \State \textcolor{blue!80!black}{$\tilde{\mathbf{d}}^{(t)} \gets \bigl(\widehat{\mathbf{M}}^{(t)}\bigr)^{-1}\mathbf{d}^{(t)}$} \hfill \textcolor{blue!80!black}{\textit{// step 3: precondition direction}}
        \State $\boldsymbol{\theta}^{(t+1)} \gets \mathrm{Adam}\!\bigl(\tilde{\mathbf{d}}^{(t)}\bigr)$ \hfill // standard Adam on preconditioned direction
    \EndFor
\end{algorithmic}
\end{algorithm}

\section{Experiments}
\label{sec:experiments}

\paragraph{Benchmarks.} We evaluate \method{} across four MOO regimes that probe distinct sources of preconditioning mismatch.
\textbf{Multi-task learning (MTL)} on $\mathtt{MultiMNIST}$~\citep{lecun1998mnist}, $\mathtt{SARCOS}$~\citep{vijayakumar2000locally}, $\mathtt{UTKFace}$~\citep{zhang2017age}, $\mathtt{Cityscapes}$~\citep{cordts2016cityscapes}, and $\mathtt{NYUv2}$~\citep{silberman2012indoor}, spanning $C\!\in\![2,7]$ tasks across classification, regression, dense prediction, and surface-normal estimation.
\textbf{Pareto MTL} on the same five datasets under preference-conditioned, time-varying $\boldsymbol{\lambda}^{(t)}$, where the goal is to trace the Pareto front rather than reach a single trade-off.
\textbf{Physics-informed Neural Networks (PINNs)} on the PINNacle benchmark~\citep{hao2024pinnacle} with $20$ PDEs whose residual / boundary / initial-condition losses live on widely different scales.
\textbf{Medical image analysis (MIA)} on $\mathtt{ISIC2018}$~\citep{codella2019skin} skin-lesion segmentation, where soft-Dice and pixel-wise cross-entropy form the two objectives, and a super-resolution task on $\mathtt{OASIS3}$~\citep{lamontagne2019oasis}, where a $4\times$ downsampled image is reconstructed to its original resolution using multiple reconstruction objectives. Per-dataset multi-objectives, backbones, losses, and metrics are in App.~\ref{app:experiment-setup}.

\paragraph{Baselines.} We apply \method{} on top of representative MOO solvers and training recipes from each regime, holding all other settings fixed within a baseline/\method{} block. For both MTL and Pareto MTL we follow PaLoRA's recipe (learning rate, backbone, batch size, schedule, per-objective losses) and compare against the standard MTL solvers linear scalarization (LS), Uncertainty Weighting (UW)~\citep{kendall2018uncertainty}, Dynamic Weight Average (DWA)~\citep{liu2019dwa}, Impartial MTL (IMTL)~\citep{liu2021imtl}, and CAGrad~\citep{liu2021conflict} for the single-trade-off setting, and against the Pareto-front solvers PaMaL~\citep{dimitriadis2023pamal} and PaLoRA~\citep{dimitriadis2025palora} for the preference-conditioned setting. For PINNs we span PINNacle's three baseline families: Vanilla (PINN), Loss Reweighting/Sampling (LRA~\citep{wang2021understanding}, NTK~\citep{wang2022ntk}, RAR~\citep{wu2023comprehensive}), and adaptive activations (LAAF/GAAF~\citep{jagtap2020adaptive}).  For skin-lesion segmentation  we compare \method{} against Adam under UNet~\citep{ronneberger2015u} and SwinTransformer~\citep{liu2021swin}. For super-resolution we compare against UW and LS under Adam. All experiments are implemented in PyTorch~\citep{paszke2019pytorch}.

\paragraph{MTL results}
\method{} improves the base solver in aggregate on every MTL benchmark (Tabs.~\ref{tab:multimnist-mtl}--\ref{tab:cityscapes-mtl} and App.~\ref{app:nyuv2-mtl}): it dominates all three task accuracies on $\mathtt{MultiMNIST}$, most per-joint errors on $\mathtt{SARCOS}$, both segmentation and depth on $\mathtt{Cityscapes}$ (IMTL relative depth error $74.2\to61.8$), and most metrics on $\mathtt{NYUv2}$. The remaining exceptions arise when the base solver already balances tasks well, with gaps within one standard deviation. On recent solvers, \method{} lowers the $\mathtt{SARCOS}$ average error of Nash-MTL~\citep{navon2022multi}, FairGrad~\citep{ban2024fair}, and ExcessMTL~\citep{he2024robust} and improves $20$ of $21$ task-wise means (App.~\ref{app:recent-solvers}).

\begin{table}[!t]
\centering
\small
\setlength{\tabcolsep}{4pt}
\caption{$\mathtt{MultiMNIST}$ and $\mathtt{UTKFace}$ MTL results (mean $\pm$ std over three seeds). Bold marks the better value within each \mbox{Base/+\method{}} pair. $\mathtt{MultiMNIST}$ rows are top-1 accuracy ($\uparrow$); $\mathtt{UTKFace}$ Age is MAE in years ($\downarrow$), Gender/Race are accuracy ($\uparrow$).}
\label{tab:multimnist-mtl}
\label{tab:utkface-mtl}
\resizebox{\linewidth}{!}{%
\begin{tabular}{lcGcGcGcGcG}
\toprule
& \multicolumn{2}{c}{LS} & \multicolumn{2}{c}{DWA~\citep{liu2019dwa}} & \multicolumn{2}{c}{UW~\citep{kendall2018uncertainty}} & \multicolumn{2}{c}{IMTL~\citep{liu2021imtl}} & \multicolumn{2}{c}{CAGrad~\citep{liu2021conflict}} \\
\cmidrule(lr){2-3}\cmidrule(lr){4-5}\cmidrule(lr){6-7}\cmidrule(lr){8-9}\cmidrule(lr){10-11}
& Base & +\method{} & Base & +\method{} & Base & +\method{} & Base & +\method{} & Base & +\method{} \\
\midrule
\multicolumn{11}{l}{\textit{$\mathtt{MultiMNIST}$}} \\
Avg.\ acc.\
  & 94.08{\scriptsize$\pm$0.37} & \textbf{95.11{\scriptsize$\pm$0.32}}
  & 94.08{\scriptsize$\pm$0.34} & \textbf{95.08{\scriptsize$\pm$0.25}}
  & 94.34{\scriptsize$\pm$0.31} & \textbf{95.14{\scriptsize$\pm$0.44}}
  & 93.95{\scriptsize$\pm$0.29} & \textbf{95.12{\scriptsize$\pm$0.25}}
  & 94.05{\scriptsize$\pm$0.24} & \textbf{95.10{\scriptsize$\pm$0.35}} \\
TL acc.\
  & 94.77{\scriptsize$\pm$0.34} & \textbf{95.76{\scriptsize$\pm$0.38}}
  & 94.71{\scriptsize$\pm$0.36} & \textbf{95.75{\scriptsize$\pm$0.31}}
  & 95.12{\scriptsize$\pm$0.46} & \textbf{95.93{\scriptsize$\pm$0.51}}
  & 94.76{\scriptsize$\pm$0.26} & \textbf{95.83{\scriptsize$\pm$0.20}}
  & 94.86{\scriptsize$\pm$0.23} & \textbf{95.77{\scriptsize$\pm$0.28}} \\
BR acc.\
  & 93.40{\scriptsize$\pm$0.75} & \textbf{94.47{\scriptsize$\pm$0.27}}
  & 93.45{\scriptsize$\pm$0.74} & \textbf{94.42{\scriptsize$\pm$0.19}}
  & 93.57{\scriptsize$\pm$0.86} & \textbf{94.35{\scriptsize$\pm$0.38}}
  & 93.15{\scriptsize$\pm$0.68} & \textbf{94.40{\scriptsize$\pm$0.33}}
  & 93.24{\scriptsize$\pm$0.71} & \textbf{94.42{\scriptsize$\pm$0.42}} \\
\midrule
\multicolumn{11}{l}{\textit{$\mathtt{UTKFace}$}} \\
Age MAE
  & 9.14{\scriptsize$\pm$0.38} & \textbf{8.77{\scriptsize$\pm$0.25}}
  & 9.70{\scriptsize$\pm$1.03} & \textbf{8.91{\scriptsize$\pm$0.19}}
  & 8.81{\scriptsize$\pm$0.42} & \textbf{8.62{\scriptsize$\pm$0.39}}
  & 8.79{\scriptsize$\pm$0.66} & \textbf{8.43{\scriptsize$\pm$0.14}}
  & \textbf{8.28{\scriptsize$\pm$0.14}} & 8.61{\scriptsize$\pm$3.80} \\
Gender Acc
  & 90.76{\scriptsize$\pm$0.26} & \textbf{91.08{\scriptsize$\pm$0.26}}
  & 90.58{\scriptsize$\pm$0.62} & \textbf{90.83{\scriptsize$\pm$0.40}}
  & 91.14{\scriptsize$\pm$0.24} & \textbf{91.45{\scriptsize$\pm$0.24}}
  & 90.94{\scriptsize$\pm$0.21} & \textbf{91.05{\scriptsize$\pm$0.39}}
  & \textbf{90.99{\scriptsize$\pm$0.22}} & 90.42{\scriptsize$\pm$0.51} \\
Race Acc
  & 81.60{\scriptsize$\pm$0.55} & \textbf{81.92{\scriptsize$\pm$0.31}}
  & 81.27{\scriptsize$\pm$0.52} & \textbf{81.52{\scriptsize$\pm$0.34}}
  & 81.22{\scriptsize$\pm$0.45} & \textbf{81.29{\scriptsize$\pm$0.27}}
  & 80.75{\scriptsize$\pm$0.40} & 80.75{\scriptsize$\pm$0.60}
  & 81.34{\scriptsize$\pm$0.36} & \textbf{81.37{\scriptsize$\pm$2.60}} \\
\bottomrule
\end{tabular}
}
\end{table}

\begin{table}[!t]
\centering
\scriptsize
\setlength{\tabcolsep}{3pt}
\caption{$\mathtt{SARCOS}$ MTL results (mean $\pm$ std over three seeds). Errors are scaled by $100$. Bold marks the better value within each \mbox{Base/+\method{}} block.}
\label{tab:sarcos-mtl}
\resizebox{\linewidth}{!}{%
\begin{tabular}{lrrrrrrrr}
\toprule
Method & Avg. err. $\downarrow$ & $t_1$ $\downarrow$ & $t_2$ $\downarrow$ & $t_3$ $\downarrow$ & $t_4$ $\downarrow$ & $t_5$ $\downarrow$ & $t_6$ $\downarrow$ & $t_7$ $\downarrow$ \\
\midrule
LS & 10.99 $\pm$ 0.59 & 1.63 $\pm$ 0.07 & 14.87 $\pm$ 0.15 & 0.65 $\pm$ 0.00 & 0.19 $\pm$ 0.00 & 50.32 $\pm$ 4.08 & 8.96 $\pm$ 0.30 & \textbf{0.30 $\pm$ 0.01} \\
\rlightgray
+ Madam & \textbf{10.00 $\pm$ 0.07} & \textbf{1.46 $\pm$ 0.02} & \textbf{13.55 $\pm$ 0.06} & \textbf{0.61 $\pm$ 0.01} & \textbf{0.19 $\pm$ 0.02} & \textbf{45.43 $\pm$ 0.16} & \textbf{8.47 $\pm$ 0.35} & 0.31 $\pm$ 0.02 \\
\midrule
DWA & 11.11 $\pm$ 0.65 & 1.64 $\pm$ 0.06 & 14.85 $\pm$ 0.06 & 0.66 $\pm$ 0.01 & 0.19 $\pm$ 0.01 & 51.15 $\pm$ 4.54 & 8.98 $\pm$ 0.10 & 0.31 $\pm$ 0.02 \\
\rlightgray
+ Madam & \textbf{10.09 $\pm$ 0.13} & \textbf{1.47 $\pm$ 0.02} & \textbf{13.57 $\pm$ 0.41} & \textbf{0.62 $\pm$ 0.02} & \textbf{0.19 $\pm$ 0.01} & \textbf{45.94 $\pm$ 0.96} & \textbf{8.51 $\pm$ 0.35} & \textbf{0.30 $\pm$ 0.02} \\
\midrule
UW & 10.83 $\pm$ 0.71 & 1.45 $\pm$ 0.02 & 14.29 $\pm$ 0.33 & 0.61 $\pm$ 0.01 & 0.18 $\pm$ 0.00 & 50.71 $\pm$ 4.64 & 8.27 $\pm$ 0.14 & 0.30 $\pm$ 0.01 \\
\rlightgray
+ Madam & \textbf{9.92 $\pm$ 0.31} & \textbf{1.33 $\pm$ 0.02} & \textbf{13.19 $\pm$ 0.38} & \textbf{0.59 $\pm$ 0.03} & \textbf{0.17 $\pm$ 0.02} & \textbf{46.14 $\pm$ 1.44} & \textbf{7.73 $\pm$ 0.34} & \textbf{0.29 $\pm$ 0.02} \\
\midrule
IMTL & 12.10 $\pm$ 0.29 & 0.84 $\pm$ 0.02 & 14.86 $\pm$ 0.16 & 0.39 $\pm$ 0.04 & \textbf{0.08 $\pm$ 0.00} & 60.69 $\pm$ 1.94 & 7.70 $\pm$ 0.19 & 0.15 $\pm$ 0.00 \\
\rlightgray
+ Madam & \textbf{11.05 $\pm$ 0.39} & \textbf{0.81 $\pm$ 0.03} & \textbf{13.32 $\pm$ 0.11} & \textbf{0.34 $\pm$ 0.01} & 0.08 $\pm$ 0.01 & \textbf{55.64 $\pm$ 2.86} & \textbf{7.02 $\pm$ 0.27} & \textbf{0.15 $\pm$ 0.02} \\
\midrule
CAGrad & 10.98 $\pm$ 0.63 & 1.16 $\pm$ 0.06 & 14.88 $\pm$ 0.53 & 0.66 $\pm$ 0.33 & 0.15 $\pm$ 0.07 & 50.30 $\pm$ 3.63 & 9.53 $\pm$ 0.55 & 0.18 $\pm$ 0.01 \\
\rlightgray
+ Madam & \textbf{10.38 $\pm$ 0.43} & \textbf{1.11 $\pm$ 0.03} & \textbf{13.79 $\pm$ 0.14} & \textbf{0.44 $\pm$ 0.02} & \textbf{0.11 $\pm$ 0.04} & \textbf{47.92 $\pm$ 3.11} & \textbf{9.13 $\pm$ 0.30} & \textbf{0.17 $\pm$ 0.00} \\
\bottomrule
\end{tabular}
}
\end{table}

\begin{table}[!t]
\centering
\small
\setlength{\tabcolsep}{4pt}
\caption{$\mathtt{Cityscapes}$ MTL mean results. Bold marks the better value within each \mbox{Base/+\method{}} pair.}
\label{tab:cityscapes-mtl}
\resizebox{\linewidth}{!}{%
\begin{tabular}{lcGcGcGcGcG|cGcG}
\toprule
& \multicolumn{2}{c}{LS} & \multicolumn{2}{c}{DWA} & \multicolumn{2}{c}{UW} & \multicolumn{2}{c}{IMTL} & \multicolumn{2}{c|}{CAGrad} & \multicolumn{2}{c}{\begin{tabular}[c]{@{}c@{}}PaMaL\\\citep{dimitriadis2023pamal}\end{tabular}} & \multicolumn{2}{c}{\begin{tabular}[c]{@{}c@{}}PaLoRA\\\citep{dimitriadis2025palora}\end{tabular}} \\
\cmidrule(lr){2-3}\cmidrule(lr){4-5}\cmidrule(lr){6-7}\cmidrule(lr){8-9}\cmidrule(lr){10-11}\cmidrule(lr){12-13}\cmidrule(lr){14-15}
& Base & +\method{} & Base & +\method{} & Base & +\method{} & Base & +\method{} & Base & +\method{} & Base & +\method{} & Base & +\method{} \\
\midrule
mIoU $\uparrow$
  & 70.12 & \textbf{70.89}
  & 70.10 & \textbf{70.86}
  & 70.20 & \textbf{71.12}
  & \textbf{70.77} & 70.49
  & 69.23 & \textbf{70.71}
  & 70.35 & \textbf{70.59}
  & 71.11 & \textbf{71.34} \\
Pix Acc $\uparrow$
  & 91.90 & \textbf{92.14}
  & 91.89 & \textbf{92.18}
  & 91.93 & \textbf{92.26}
  & \textbf{92.12} & 92.02
  & 91.61 & \textbf{91.63}
  & 91.99 & \textbf{92.17}
  & 92.21 & \textbf{92.29} \\
Abs Err $\downarrow$
  & 0.0192 & \textbf{0.0187}
  & 0.0192 & \textbf{0.0184}
  & 0.0189 & \textbf{0.0172}
  & 0.0151 & \textbf{0.0147}
  & 0.0168 & \textbf{0.0165}
  & \textbf{0.0141} & 0.0143
  & \textbf{0.014} & 0.0147 \\
Rel Err $\downarrow$
  & 124.06 & \textbf{123.17}
  & 127.66 & \textbf{120.31}
  & 125.94 & \textbf{116.10}
  & 74.23 & \textbf{61.81}
  & 110.14 & \textbf{101.95}
  & \textbf{54.52} & 54.99
  & \textbf{51.27} & 52.27 \\
\bottomrule
\end{tabular}
}
\end{table}

\paragraph{Pareto MTL results}
On $\mathtt{MultiMNIST}$ (Fig.~\ref{fig:multimnist-pareto-mtl}), $\mathtt{SARCOS}$ (Tab.~\ref{tab:sarcos-pareto-mtl}), $\mathtt{Cityscapes}$ (Tab.~\ref{tab:cityscapes-mtl}), and $\mathtt{NYUv2}$ (Tab.~\ref{tab:nyuv2-mtl}). \method{} raises hypervolume and improves both PaLoRA and PaMaL across all benchmarks: it dominates the base front on $\mathtt{MultiMNIST}$, reduces per-ray error on nearly every $\mathtt{SARCOS}$ ray, improves both segmentation metrics on $\mathtt{Cityscapes}$ with only marginal depth trade-offs, and yields gains on a majority of $\mathtt{NYUv2}$ metrics, uniformly on surface normals. 

\begin{figure}[!t]
\centering
\begin{subfigure}[t]{0.32\linewidth}
\centering
\includegraphics[width=\linewidth]{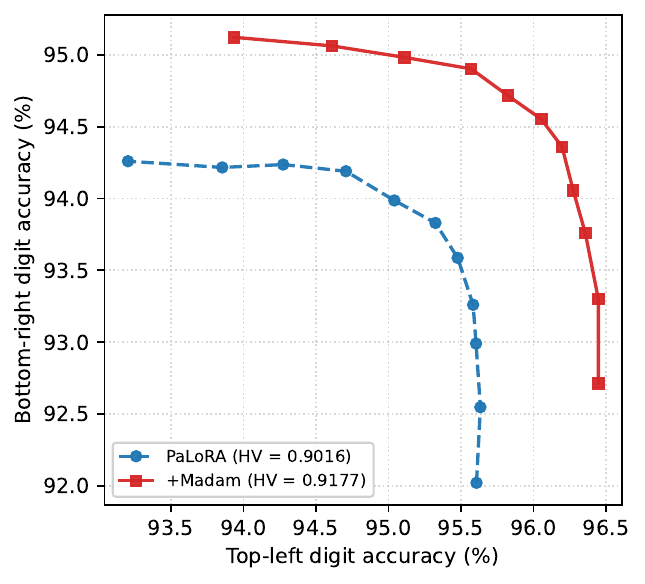}
\label{fig:multimnist-pareto-palora}
\end{subfigure}
\hfill
\begin{subfigure}[t]{0.32\linewidth}
\centering
\includegraphics[width=\linewidth]{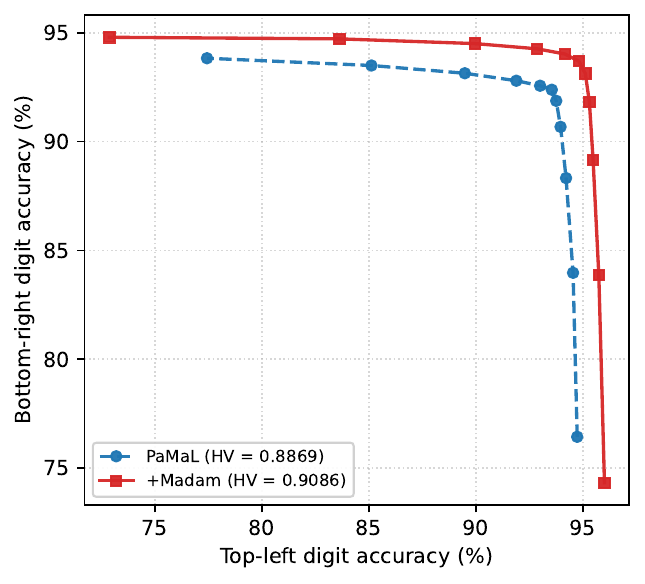}
\label{fig:multimnist-pareto-pamal}
\end{subfigure}
\hfill
\begin{subfigure}[t]{0.32\linewidth}
\centering
\includegraphics[width=\linewidth]{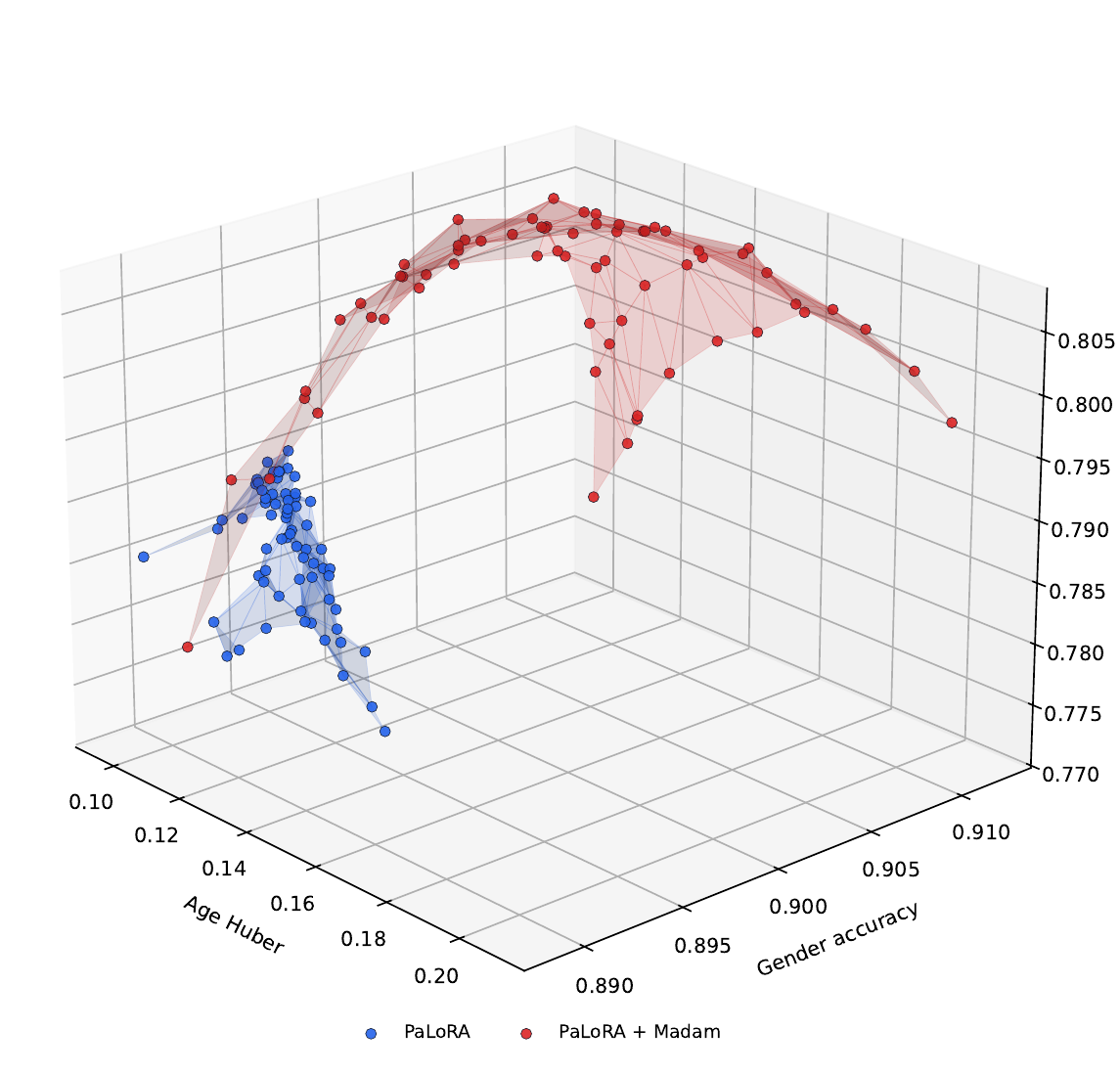}
\label{fig:utkface-pareto-palora}
\end{subfigure}
\caption{Pareto fronts on $\mathtt{MultiMNIST}$ (left, center) and $\mathtt{UTKFace}$ (right). $\mathtt{MultiMNIST}$ shows the two-task accuracy space (top-left vs.\ bottom-right digit) under PaLoRA and PaMaL, evaluated on $11$ uniformly spaced preference rays. $\mathtt{UTKFace}$ shows the $3$-D metric front under PaLoRA. Per-ray per-objective numbers for $\mathtt{MultiMNIST}$ are tabulated in App.~\ref{app:multimnist-pareto-per-objective} (Tab.~\ref{tab:multimnist-pareto-mtl-per-objective}).}
\label{fig:multimnist-pareto-mtl}
\end{figure}

\begin{table}[!t]
\centering
\scriptsize
\setlength{\tabcolsep}{3pt}
\caption{$\mathtt{SARCOS}$ Pareto MTL results (mean $\pm$ std over three seeds). Errors are scaled by $100$; $r_0$--$r_7$ report per-ray average error. Bold marks the better value within each \mbox{Base/+\method{}} block. Per-task per-ray numbers are tabulated in App.~\ref{app:multimnist-pareto-per-objective} (Tabs.~\ref{tab:sarcos-pareto-mtl-palora-per-objective}, \ref{tab:sarcos-pareto-mtl-pamal-per-objective}).}
\label{tab:sarcos-pareto-mtl}
\resizebox{\linewidth}{!}{%
\begin{tabular}{lrrrrrrrrrr}
\toprule
Method & HV $\uparrow$ & Mean $\downarrow$ & $r_0$ $\downarrow$ & $r_1$ $\downarrow$ & $r_2$ $\downarrow$ & $r_3$ $\downarrow$ & $r_4$ $\downarrow$ & $r_5$ $\downarrow$ & $r_6$ $\downarrow$ & $r_7$ $\downarrow$ \\
\midrule
PaMaL & 378.2864 $\pm$ 9.2692 & 48.31 $\pm$ 17.53 & 87.86 $\pm$ 87.91 & 41.22 $\pm$ 9.52 & \textbf{34.91 $\pm$ 3.77} & 37.00 $\pm$ 6.33 & 61.62 $\pm$ 40.73 & 40.42 $\pm$ 5.35 & 55.28 $\pm$ 13.61 & 28.19 $\pm$ 0.85 \\
\rlightgray
+ Madam & \textbf{394.5543 $\pm$ 4.1958} & \textbf{38.48 $\pm$ 10.98} & \textbf{32.06 $\pm$ 2.72} & \textbf{33.05 $\pm$ 1.46} & 86.15 $\pm$ 89.88 & \textbf{35.31 $\pm$ 5.35} & \textbf{32.92 $\pm$ 0.52} & \textbf{30.59 $\pm$ 1.23} & \textbf{32.72 $\pm$ 2.08} & \textbf{25.01 $\pm$ 0.74} \\
\midrule
PaLoRA & 431.2050 $\pm$ 8.4737 & 16.62 $\pm$ 3.61 & 16.53 $\pm$ 2.20 & 16.01 $\pm$ 2.37 & 18.31 $\pm$ 7.37 & 15.53 $\pm$ 1.67 & 15.47 $\pm$ 2.02 & 18.86 $\pm$ 7.90 & 18.52 $\pm$ 5.83 & 13.76 $\pm$ 1.19 \\
\rlightgray
+ Madam & \textbf{442.1679 $\pm$ 3.4060} & \textbf{13.91 $\pm$ 1.64} & \textbf{14.83 $\pm$ 2.40} & \textbf{13.96 $\pm$ 2.03} & \textbf{13.53 $\pm$ 2.24} & \textbf{12.45 $\pm$ 0.92} & \textbf{14.05 $\pm$ 1.97} & \textbf{15.28 $\pm$ 1.89} & \textbf{14.88 $\pm$ 4.39} & \textbf{12.31 $\pm$ 0.66} \\
\bottomrule
\end{tabular}
}
\end{table}

\paragraph{PINN results}
Tabs.~\ref{tb:pinn-vanilla}--\ref{tb:pinn-architecture} (App.~\ref{app:pinn-detailed}) report mean L2RE on PINNacle for three baseline families, each paired with \method{}: Vanilla (PINN); Loss Reweighting/Sampling (LRA, NTK, RAR); and adaptive activations (LAAF, GAAF). \method{} improves the base solver on most PDEs in every family, with the largest gains on geometrically complex domains, chaotic systems, and high-dimensional Poisson, where residual, boundary, and initial-condition losses span disparate scales; gains compound when \method{} is composed with loss-balancing baselines, while the few cases favoring the base lie in regimes where both saturate near $100\%$ L2RE.

\paragraph{MIA results}
On $\mathtt{OASIS3}$ brain-MRI $4\times$ super-resolution (Fig.~\ref{fig:oasis3-combined}), pairing \method{} with the LS reweighting scheme yields consistent improvements in perceptual fidelity, as measured by LPIPS~\citep{zhang2018unreasonable} and DISTS~\citep{ding2020image}, relative to the L1 baseline. These gains extend to the pixel-level metrics (PSNR, SSIM, MAE) as well, so perceptual and pixelwise quality improve jointly; qualitatively, this manifests as sharper preservation of fine cortical structure. On $\mathtt{ISIC2018}$ segmentation (Fig.~\ref{fig:oasis3-combined}, bottom-left), \method{} performs on par with Adam under the SwinTransformer backbone, and combining it with RLW~\citep{lin2022reasonable} under UNet produces the most balanced Dice/HD95 profile among the configurations considered.

\begin{figure}[!t]
\centering
\begin{minipage}[c]{0.36\linewidth}
\centering
{\scriptsize $\mathtt{OASIS3}$ super-resolution metric}\\[2pt]
\scriptsize
\resizebox{\linewidth}{!}{%
\begin{tabular}{@{}lcG@{}}
\toprule
Metric & L1 baseline & LS + \method{} \\
\midrule
LPIPS \down       & \pmv{0.2823}{0.020}& \best{\pmv{0.2755}{0.019}} \\
DISTS \down       & \pmv{0.2421}{0.014}& \best{\pmv{0.2359}{0.014}} \\
MAE \down         & \pmv{0.0425}{0.004}& \best{\pmv{0.0405}{0.004}} \\
PSNR (dB) \up     & \pmv{21.244}{0.463}& \best{\pmv{21.522}{0.491}} \\
SSIM \up          & \pmv{0.8095}{0.019}& \best{\pmv{0.8181}{0.019}} \\
\bottomrule
\end{tabular}
}

\vspace{6pt}
{\scriptsize $\mathtt{ISIC2018}$ segmentation: Dice $\uparrow$, HD95 $\downarrow$}\\[2pt]
\resizebox{\linewidth}{!}{%
\begin{tabular}{lcccc}
\toprule
\multirow{2}{*}{Method} & \multicolumn{2}{c}{Swin} & \multicolumn{2}{c}{UNet} \\
\cmidrule(lr){2-3} \cmidrule(lr){4-5}
 & Dice & HD95 & Dice & HD95 \\
\midrule
LS            & 0.8849        & 20.6307        & 0.8775          & 21.1575 \\
\rlightgray
LS + \method{}  & \best{0.8855} & \best{20.1855} & 0.8772          & 22.1048 \\
\rlightgray
RLW + \method{} & 0.8807        & 20.8601        & \textbf{0.8821} & \textbf{20.9142} \\
\bottomrule
\end{tabular}
}
\end{minipage}%
\hfill
\begin{minipage}[c]{0.62\linewidth}
\centering
\includegraphics[width=\linewidth]{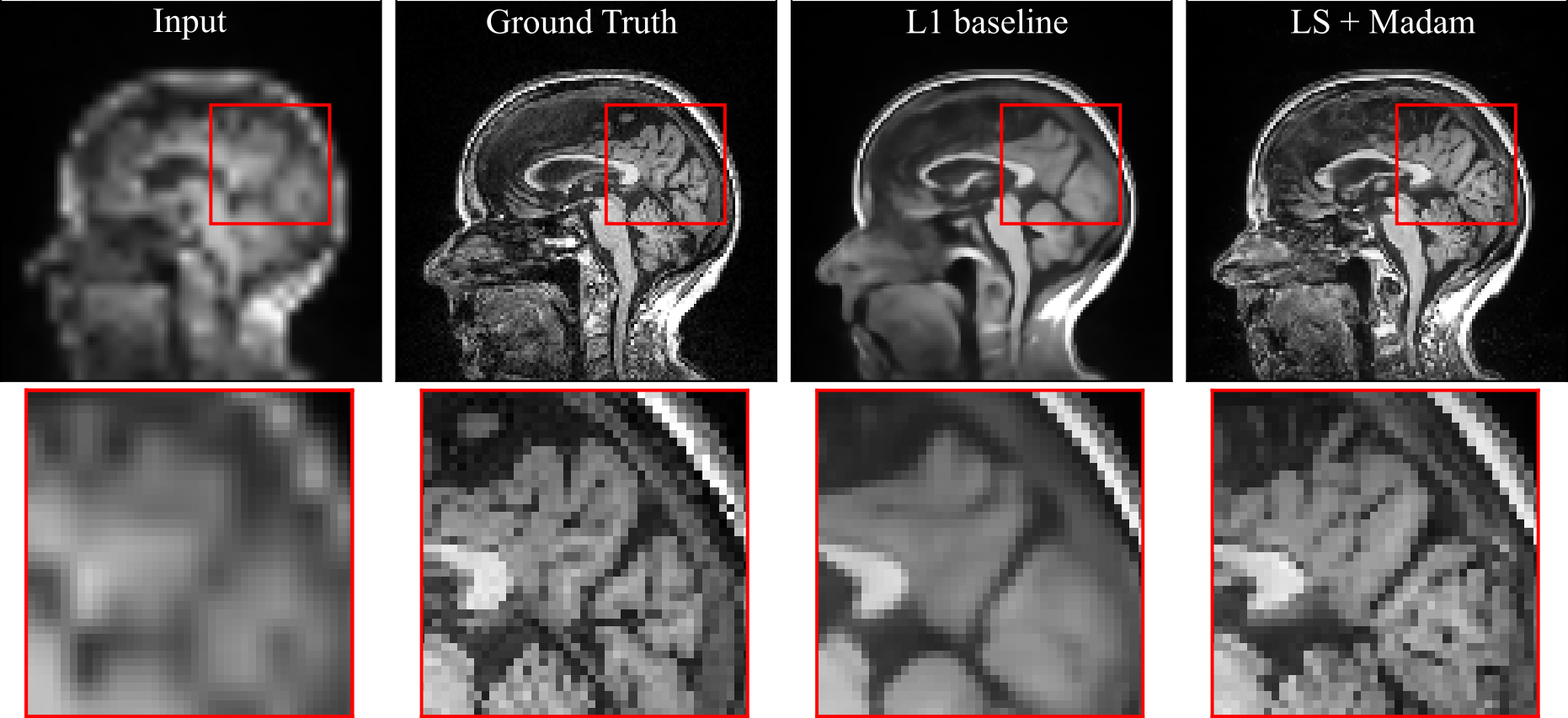}
\end{minipage}
\caption{MIA results. \textbf{Left top:} $\mathtt{OASIS3}$ super-resolution validation metrics (mean $\pm$ std over three seeds; LPIPS~\citep{zhang2018unreasonable} and DISTS~\citep{ding2020image} are perceptual, PSNR/SSIM/MAE are pixelwise; bold marks the better value within each \mbox{Base/+\method{}} block). \textbf{Left bottom:} $\mathtt{ISIC2018}$ skin-lesion validation Dice ($\uparrow$) and HD95 ($\downarrow$) across backbones; bold marks the better value within each \mbox{Adam/+\method{}} block. 
\textbf{Right:} qualitative $\mathtt{OASIS3}$ reconstruction results: the zoom-in region outlined in red highlights LS + \method{} better preserves the cortical folds.}
\label{fig:oasis3-combined}
\end{figure}

\paragraph{Ablation studies}
We deactivate each of \method{}'s three components in turn under LS on $\mathtt{SARCOS}$ (Tab.~\ref{tab:sarcos-ablation}). Disabling the rampup is most damaging: engaging the preconditioner before per-objective Fisher estimates stabilize drops performance below the LS baseline. Full quadratic enumeration over objective pairs matches the full method at quadratic cost, validating stochastic sampling as an unbiased estimator. Restricting to the diagonal Fisher similarly erodes the gain, indicating off-diagonal curvature between conflicting objectives contributes meaningfully. The full method attains the strongest aggregate performance and lowest variance across seeds. App.~\ref{app:curvature-baselines} further compares \method{} against curvature-aware optimizers (FAdam, Shampoo, AdaHessian) under the same setup.

\begin{table}[!t]
\centering
\scriptsize
\setlength{\tabcolsep}{3pt}
\caption{$\mathtt{SARCOS}$ ablation under LS scalarization (mean $\pm$ std over three seeds). Errors are scaled by $100$. Bold marks the best value in each column.}
\label{tab:sarcos-ablation}
\resizebox{\linewidth}{!}{%
\begin{tabular}{lrrrrrrrr}
\toprule
Method & Avg. err. $\downarrow$ & $t_1$ $\downarrow$ & $t_2$ $\downarrow$ & $t_3$ $\downarrow$ & $t_4$ $\downarrow$ & $t_5$ $\downarrow$ & $t_6$ $\downarrow$ & $t_7$ $\downarrow$ \\
\midrule
\rlightgray
LS + \method{} (full)               & \textbf{10.00 $\pm$ 0.07} & 1.46 $\pm$ 0.02 & \textbf{13.55 $\pm$ 0.06} & \textbf{0.61 $\pm$ 0.01} & 0.19 $\pm$ 0.02 & \textbf{45.43 $\pm$ 0.16} & 8.47 $\pm$ 0.35 & 0.31 $\pm$ 0.02 \\
\;\;w/o FIM rampup Eq.~\eqref{eq:practical-metric} & 11.74 $\pm$ 0.54 & 1.76 $\pm$ 0.03 & 16.43 $\pm$ 0.97 & 0.72 $\pm$ 0.07 & 0.22 $\pm$ 0.03 & 52.85 $\pm$ 2.27 & 9.88 $\pm$ 0.62 & 0.34 $\pm$ 0.02 \\
\;\;w/o stochastic pair sampling    & 10.33 $\pm$ 0.61 & \textbf{1.45 $\pm$ 0.04} & 14.02 $\pm$ 0.83 & 0.62 $\pm$ 0.03 & \textbf{0.18 $\pm$ 0.01} & 47.44 $\pm$ 4.53 & \textbf{8.32 $\pm$ 0.25} & \textbf{0.30 $\pm$ 0.00} \\
\;\;w/o off-diagonal $\mathbf{F}_{ij}$        & 10.47 $\pm$ 0.70 & 1.50 $\pm$ 0.02 & 14.18 $\pm$ 0.35 & 0.64 $\pm$ 0.03 & 0.18 $\pm$ 0.01 & 48.10 $\pm$ 4.93 & 8.35 $\pm$ 0.25 & 0.30 $\pm$ 0.01 \\
\midrule
LS (baseline)                       & 10.99 $\pm$ 0.59 & 1.63 $\pm$ 0.07 & 14.87 $\pm$ 0.15 & 0.65 $\pm$ 0.00 & 0.19 $\pm$ 0.00 & 50.32 $\pm$ 4.09 & 8.96 $\pm$ 0.30 & 0.30 $\pm$ 0.01 \\
\bottomrule
\end{tabular}
}
\end{table}

\paragraph{Runtime and trajectory analysis} On $\mathtt{SARCOS}$, \method{} increases per-step wall-clock by $12$--$19\%$ due to two extra sampled backward passes and $O(C^2)$ Fisher EMA storage, yet CAGrad + \method{} reaches a matched validation error with $20\%$ fewer cumulative FLOPs than CAGrad + Adam (App.~\ref{app:cost}). On a two-dimensional toy problem whose correct local metric rotates with $\boldsymbol{\lambda}$, \method{} follows the rotating stiff direction toward $\mathbf{x}^\star(\boldsymbol{\lambda})$, whereas Adam stays locked to the ambient basis (Fig.~\ref{fig:madam-trajectory}, App.~\ref{app:madam-trajectory}).

\section{Conclusion}
\label{sec:conclusion}

We introduced \method{}, a drop-in wrapper that resolves two mismatches in the solver--Adam pipeline: a \emph{weighting mismatch} with Adam's history-averaged second moment, and a \emph{geometric mismatch} with the Euclidean geometry assumed by MOO solvers. Preconditioning the reconciled direction by the preference-conditioned diagonal Fisher lets this metric govern the update. Across multi-task learning, Pareto-front recovery, PINNs, and medical imaging, \method{} consistently improves over Adam for diverse MOO solvers.

\paragraph{Limitations and future work.}
The main limitation of \method{} is memory: it stores $C(C+1)/2$ cross-objective Fisher EMAs over the shared parameters, so its optimizer state grows quadratically in the number of objectives. Scaling \method{} to many more objectives and to larger models, for example by restricting the stored pairs or sharding the Fisher bank, is left to future work.

\begin{ack}
This work was partially supported by funding from NewYork-Presbyterian for the NYP-Cornell Cardiovascular AI Collaboration.
We also gratefully acknowledge use of the research computing resources of the Empire AI Consortium, Inc, with support from Empire State Development of the State of New York, the Simons Foundation, and the Secunda Family Foundation~\citep{bloom2025empireai}.
\end{ack}
\bibliographystyle{abbrvnat}
\bibliography{references}

\appendix
\newpage
\MarkAppendixInTOC
\setcounter{tocdepth}{2}
\AppendixOnlyTOC
\clearpage
\section{Comprehensive Related Work}\label{app:related-work-comprehensive}

\paragraph{Multi-Task and Pareto MTL.}
MTL methods reconcile multiple objectives by reshaping per-objective loss weights, projecting or rotating task gradients, or steering toward a preference-conditioned point on the trade-off front, and then hand the resulting direction to a generic optimizer.
The baselines we compare against in experiments are Linear Scalarization (LS), the equal-weight sum;
Uncertainty Weighting~\citep{kendall2018uncertainty}, which derives weights from homoscedastic noise;
Dynamic Weight Average~\citep{liu2019dwa}, which reweights by recent loss ratios;
IMTL~\citep{liu2021imtl}, which normalizes loss scale (IMTL-L) and per-objective gradient norms (IMTL-G);
CAGrad~\citep{liu2021conflict}, which maximizes the worst-case task improvement;
PaMaL~\citep{dimitriadis2023pamal}, which builds Pareto solutions as preference-weighted ensembles of single-task models;
and PaLoRA~\citep{dimitriadis2025palora}, which extends preference conditioning to low-rank adapters so that varying $\boldsymbol{\lambda}$ at inference traces the entire Pareto front.
Beyond these, the broader category includes loss-balancing methods such as GradNorm~\citep{chen2018gradnorm}, FAMO~\citep{liu2023famo}, Smooth Tchebycheff scalarization~\citep{lin2024smooth}, DB-MTL~\citep{lin2023dbmtl}.
gradient-balancing methods such as PCGrad~\citep{yu2020gradient}, GradVac~\citep{wang2021gradvac}, Nash-MTL~\citep{navon2022multi}, RotoGrad~\citep{javaloy2022rotograd}, Aligned-MTL~\citep{senushkin2023aligned}, GradDrop~\citep{soen2024trade}, and TAG~\citep{fifty2021tag};
multi-task optimizer variants MTAdam~\citep{malkiel2021mtadam} and AdaTask~\citep{yang2023adatask};
and additional Pareto-front solvers MGDA~\citep{sener2018multi}, PMTL~\citep{lin2019pareto}, EPO~\citep{mahapatra2020epo}, COSMOS~\citep{ruchte2021scalable}, and FERERO~\citep{chen2024ferero}.
Across all of these, the constructed direction is then committed by Adam, whose diagonal second-moment EMA conflates the time-varying preference $\boldsymbol{\lambda}^{(t)}$ with per-objective gradient statistics and imposes a task-agnostic RMS metric on the executed step, so the solver's prescribed weighting and direction are distorted by the metric of the optimizer that every method here defaults to. A concurrent study by \citet{zhou2026delve} also identifies an incompatibility between MTL-derived directions and adaptive optimizers, but the two works differ in mechanism, optimizer scope, and evaluation. \citet{zhou2026delve} observe that under adaptive optimizers the instantaneous MTL direction plays only a marginal role in the actual parameter update, because the first-moment EMA lets accumulated momentum dominate; their APT framework therefore adapts the momentum coefficient $\beta_1$ so that the current aggregated gradient regains influence. \method{} targets the second moment instead: Adam's second-moment EMA mixes past solver weights into the curvature estimate (Prop.~\ref{prop:adam-weighting-mismatch}) and distorts the Euclidean geometry the solver assumes (Prop.~\ref{prop:adam-geometry}), and \method{} reconstructs a preference-conditioned diagonal Fisher from cross-objective moments before the direction reaches Adam. The two corrections act on different EMAs and are in principle composable. In optimizer scope, \citet{zhou2026delve} study Adam, AdamW, and Muon, and add a separate direction-preservation mechanism for Muon's orthogonalization; \method{} is an Adam wrapper that leaves the optimizer's internal update unchanged and does not modify or evaluate Muon. In evaluation, \citet{zhou2026delve} evaluate single-solution MTL under a fixed preference; \method{} additionally evaluates preference-conditioned Pareto-front recovery under time-varying $\boldsymbol{\lambda}^{(t)}$, where the weighting mismatch of Prop.~\ref{prop:adam-weighting-mismatch} is most pronounced, and covers loss-balancing, gradient-balancing, and Pareto-based solvers across MTL, PINNs, and medical imaging.

\paragraph{PINN and MIA Loss Balancing.}
Physics-informed neural networks (PINNs) and medical-image analysis (MIA) training confront a closely related problem: residual, boundary, and initial-condition losses in PINNs, and perceptual, structural, and pixel-wise reconstruction losses in MIA, live on widely different scales, so a single objective dominates training unless explicitly rebalanced.
For PINNs we evaluate against the three PINNacle~\citep{hao2024pinnacle} families: the vanilla PINN~\citep{raissi2019physics}, which trains a network to satisfy the PDE residual and boundary/initial-condition losses under uniform weighting;
LRA~\citep{wang2021understanding}, which adapts loss weights from gradient-magnitude statistics;
NTK~\citep{wang2022ntk}, which derives per-loss weights from the Neural Tangent Kernel of each loss term;
RAR~\citep{wu2023comprehensive}, which adaptively resamples residual collocation points by residual magnitude;
and LAAF / GAAF~\citep{jagtap2020adaptive}, which alter expressivity through locally and globally adaptive activation functions rather than weights.
For MIA we compare against vanilla Adam~\citep{paszke2019pytorch} on UNet~\citep{ronneberger2015u} and SwinTransformer~\citep{liu2021swin} backbones, and against RLW~\citep{lin2022reasonable}, which periodically reweights losses by random draws.
Other methods surveyed in PINNacle, including causal training, time-marching schedules, and further sampling strategies, share the same first-order spirit and are not separately enumerated.
These methods correctly diagnose loss-imbalance and address it with first-order remedies (weights, samples, or activations), but none modifies the geometry under which the rebalanced direction is executed; the optimizer's step metric, supplied by Adam, is left untouched even when the loss landscape it defines is the source of the imbalance.

\paragraph{Curvature-Aware Optimizers.}
A separate line of work derives parameter updates from second-order or Fisher information rather than the raw gradient, aiming to accelerate convergence on a single scalar objective by replacing the first-order optimizer altogether.
None of these methods is used as a baseline in our experiments, but they form the closest neighbor to \method{} on the metric axis.
Natural gradient descent~\citep{amari1998natural,martens2020new} preconditions the update by the Fisher information matrix;
K-FAC~\citep{martens2015optimizing} approximates this Fisher in a Kronecker-factored form layer-wise, with EKFAC~\citep{george2018ekfac} extending the factorization to its eigenbasis;
Shampoo~\citep{gupta2018shampoo} maintains full preconditioners along each tensor axis, and SOAP~\citep{vyas2025soap} composes Shampoo's preconditioner with Adam's adaptive scaling;
AdaHessian~\citep{yao2021adahessian} estimates the Hessian diagonal via Hutchinson sampling, and Sophia~\citep{liu2023sophia} clips the diagonal Hessian estimate for stable language-model pretraining;
FAdam~\citep{hwang2024fadam} reinterprets Adam itself as a diagonal empirical-Fisher natural-gradient optimizer, and recent diagonal-Fisher estimators~\citep{soen2024trade,li2025fishers,wu2024improved} sharpen the surrogate Adam already implicitly uses, while \citet{kunstner2019limitations} cautions against treating Adam's denominator as a generic Fisher proxy.
All of these methods target a single scalar objective and intervene by replacing the first-order optimizer with a curvature-aware one; none preconditions the reconciled direction of a multi-objective solver under a time-varying preference.
\method{} differs in both scope and integration: it derives a preference-conditioned diagonal Fisher of the scalarized objective and slots it into the solver-then-Adam pipeline as a wrapper, leaving both the MOO solver and the optimizer untouched.

\clearpage
\section{Representative MOO Solvers}
\label{app:moo-solvers}

This appendix walks through representative gradient-based MOO solvers and verifies for each that the reconciled direction fits the linear-scalarization form of Definition~\ref{def:moo-solver},
\begin{equation}\label{eq:appendix-linear-form}
    \mathbf{d}^{(t)} \;=\; \sum_{i=1}^{C}\lambda_i^{(t)}\,\mathbf{g}_i^{(t)},
    \qquad \mathbf{g}_i^{(t)} \;:=\; \nabla_{\boldsymbol{\theta}}\ell_i\!\bigl(\boldsymbol{\theta}^{(t)}\bigr),
\end{equation}
with $\boldsymbol{\lambda}^{(t)}$ detached from $\boldsymbol{\theta}$ in the backward pass. Each per-objective gradient $\mathbf{g}_i^{(t)}$ is, by autodiff, the gradient of $\ell_i$ under the standard Euclidean inner product on $\mathbb{R}^d$. Linearity of the gradient operator with $\boldsymbol{\theta}$-detached coefficients then gives
\begin{equation}\label{eq:appendix-euclidean-gradient}
    \mathbf{d}^{(t)} \;=\; \nabla_{\boldsymbol{\theta}}\,\ell_{\boldsymbol{\lambda}^{(t)}}\!\bigl(\boldsymbol{\theta}^{(t)}\bigr),
    \qquad \ell_{\boldsymbol{\lambda}^{(t)}} \;:=\; \sum_{i=1}^C \lambda_i^{(t)}\ell_i,
\end{equation}
the Euclidean gradient of the scalarized loss on $\mathbb{R}^d$. The three families below differ only in how $\boldsymbol{\lambda}^{(t)}$ is determined; the Euclidean property is inherited uniformly, and so is the geometric mismatch with Adam's diagonal RMS metric (Sec.~\ref{sec:mismatch}).

For the gradient-balancing methods we use the per-objective gradient matrix $\mathbf{G}^{(t)} := [\mathbf{g}_1^{(t)},\ldots,\mathbf{g}_C^{(t)}]\in\mathbb{R}^{d\times C}$ and its Gram $\mathbf{K}^{(t)} := \mathbf{G}^{(t)\top}\mathbf{G}^{(t)}\in\mathbb{R}^{C\times C}$ with $[\mathbf{K}^{(t)}]_{ij} = \mathbf{g}_i^{(t)\top}\mathbf{g}_j^{(t)}$. Each Gram entry is the trace of the cross-objective Fisher diagonal $\mathbf{F}_{ij}^{(t)} := \mathbb{E}[\mathbf{g}_i^{(t)}\odot\mathbf{g}_j^{(t)}]\in\mathbb{R}^d$, namely $[\mathbf{K}^{(t)}]_{ij} = \mathbf{1}^\top\mathbf{F}_{ij}^{(t)}$ in expectation; \method{} retains the full per-coordinate $\mathbf{F}_{ij}^{(t)}$ used to assemble the preference-conditioned metric in Eq.~\eqref{eq:target-metric}.

\subsection{Gradient Balancing}
\label{app:moo-gb}

Gradient-balancing solvers determine $\boldsymbol{\lambda}^{(t)}$ implicitly from the geometry of the per-objective gradients (typically a function of $\mathbf{K}^{(t)}$). The output is a linear combination of $\{\mathbf{g}_i^{(t)}\}$ with coefficients detached from $\boldsymbol{\theta}$.

\paragraph{MGDA~\citep{sener2018multi}.}
Picks $\boldsymbol{\lambda}^{(t)}\in\Delta^{C}$ as the minimum-norm element of the convex hull of the per-objective gradients,
\begin{equation}
    \boldsymbol{\lambda}^{(t)} \;=\; \arg\min_{\boldsymbol{\lambda}\in\Delta^{C}} \boldsymbol{\lambda}^{\top}\mathbf{K}^{(t)}\boldsymbol{\lambda},
\end{equation}
solved as a small QP in $\mathbb{R}^C$. The reconciled direction $\mathbf{d}^{(t)}=\sum_i \lambda_i^{(t)}\mathbf{g}_i^{(t)}$ is by construction~\eqref{eq:appendix-linear-form}; the QP is solved with stop-gradient on $\mathbf{K}^{(t)}$, so $\boldsymbol{\lambda}^{(t)}$ is detached.

\paragraph{PCGrad~\citep{yu2020gradient}.}
For each pair $(i,j)$ with $\mathbf{g}_i^{(t)\top}\mathbf{g}_j^{(t)}<0$, PCGrad replaces $\mathbf{g}_i^{(t)}$ by its projection onto the orthogonal complement of $\mathbf{g}_j^{(t)}$ and sums the projected vectors. Each projection is linear in $\{\mathbf{g}_k^{(t)}\}$ with coefficients drawn from $\mathbf{K}^{(t)}$, so the output collects into~\eqref{eq:appendix-linear-form} with effective coefficients
\begin{equation}
    \lambda_i^{(t)} \;=\; 1 \;-\; \sum_{j\in\mathcal{C}_i^{(t)}} \frac{\mathbf{g}_j^{(t)\top}\mathbf{g}_i^{(t)}}{\bigl\|\mathbf{g}_i^{(t)}\bigr\|^2},
\end{equation}
where $\mathcal{C}_i^{(t)}$ indexes objectives that conflict with task $i$ at step $t$. Coefficients lie in $\mathbb{R}^C$ rather than $\Delta^C$, but the form is still linear scalarization with detached weights.

\paragraph{CAGrad~\citep{liu2021conflict}.}
Solves a trust-region problem around the average gradient $\mathbf{g}_0^{(t)} := \tfrac{1}{C}\sum_i \mathbf{g}_i^{(t)}$, returning
\begin{equation}
    \mathbf{d}^{(t)} \;=\; \mathbf{g}_0^{(t)} + \alpha^{*}\mathbf{G}^{(t)}\boldsymbol{w}^{*},
    \qquad \boldsymbol{w}^{*}\in\Delta^C,
\end{equation}
with $\boldsymbol{w}^{*}$ from a small QP and trust-region radius $\alpha^{*}$. Substituting $\mathbf{g}_0^{(t)} = \mathbf{G}^{(t)}(\mathbf{1}/C)$ shows $\mathbf{d}^{(t)}$ is~\eqref{eq:appendix-linear-form} with effective weights $\lambda_i^{(t)} = 1/C + \alpha^{*}w_i^{*}$.

\paragraph{Other gradient-balancing methods.}
IMTL-G~\citep{liu2021imtl} chooses $\boldsymbol{\lambda}^{(t)}$ in closed form so that the reconciled direction has equal projections onto each unit-normalized $\mathbf{g}_i^{(t)}$. RotoGrad~\citep{javaloy2022rotograd} rotates per-objective feature spaces so that per-objective gradients align in magnitude, then sums the rotated gradients (taking the rotated $\tilde{\mathbf{g}}_i^{(t)}$ as the per-objective gradients leaves the form unchanged). Aligned-MTL~\citep{senushkin2023aligned} rebalances gradients in the eigenbasis of $\mathbf{K}^{(t)}$, Nash-MTL~\citep{navon2022multi} solves a bargaining QP for $\boldsymbol{\lambda}^{(t)}$, and FairGrad~\citep{ban2024fair} chooses $\boldsymbol{\lambda}^{(t)}$ so that the per-objective improvements $\mathbf{g}_i^{(t)\top}\mathbf{d}^{(t)}$ follow an $\alpha$-fair allocation, which reduces to a nonlinear system in $\mathbf{K}^{(t)}$ and $\boldsymbol{\lambda}^{(t)}$. All produce~\eqref{eq:appendix-linear-form} with weights determined by $\mathbf{K}^{(t)}$ and detached from $\boldsymbol{\theta}$.

\subsection{Loss Balancing}
\label{app:moo-lb}

Loss-balancing solvers determine $\boldsymbol{\lambda}^{(t)}$ from the history of scalar losses $\{\ell_i^{(\tau)}\}_{\tau\le t}$ and form $\mathbf{d}^{(t)}$ via~\eqref{eq:appendix-linear-form} with $\boldsymbol{\lambda}^{(t)}$ detached from $\boldsymbol{\theta}$ in the backward pass.

\paragraph{GradNorm~\citep{chen2018gradnorm}.}
Adjusts $\boldsymbol{\lambda}^{(t)}$ by a separate optimization step that drives per-objective gradient norms toward relative training-rate targets. The main update consumes the resulting $\boldsymbol{\lambda}^{(t)}$ as linear-scalarization weights.

\paragraph{FAMO~\citep{liu2023famo}.}
Maintains $\boldsymbol{\lambda}^{(t)} = \mathrm{softmax}(\boldsymbol{z}^{(t)})$ and updates the logits $\boldsymbol{z}^{(t)}$ via a closed-form rule based on per-objective log-loss decrease rates. The reconciled direction is again~\eqref{eq:appendix-linear-form}.

\paragraph{ExcessMTL~\citep{he2024robust}.}
Sets $\boldsymbol{\lambda}^{(t)}$ from per-task excess risks, the gap between the current loss and an estimate of the task's optimal loss, so that tasks farther from convergence receive larger weights. The main update consumes $\boldsymbol{\lambda}^{(t)}$ as linear-scalarization weights.

\paragraph{Uncertainty Weighting~\citep{kendall2018uncertainty}.}
Learns per-objective observation noises $\sigma_i$ jointly with $\boldsymbol{\theta}$; the effective per-objective weight at the main step is $\lambda_i^{(t)} = 1/\bigl(2(\sigma_i^{(t)})^2\bigr)$ (modulo regularization terms in the loss). The detached weight feeds into~\eqref{eq:appendix-linear-form}.

\subsection{Preference-Based Pareto}
\label{app:moo-pb}

Preference-based Pareto methods take $\boldsymbol{\lambda}^{(t)}$ as an external input, typically a user preference on the simplex, and use it to steer the iterate toward a particular point on the Pareto front. Sweeping $\boldsymbol{\lambda}^{(t)}$ within a single run (or across runs) traces out the front.

\paragraph{Pareto MTL~\citep{lin2019pareto}.}
A small set of reference rays $\{\boldsymbol{\lambda}^{(k)}\}$ partitions the simplex into cones; for each ray, training is conducted under a constrained scalarized subproblem with $\boldsymbol{\lambda}^{(t)}$ fixed inside the cone. The descent direction at each step is~\eqref{eq:appendix-linear-form} with $\boldsymbol{\lambda}^{(t)}$ supplied by the active ray.

\paragraph{PaLoRA~\citep{dimitriadis2025palora}.}
Trains a single conditional model on per-step preferences $\boldsymbol{\lambda}^{(t)}\sim\mathrm{Dir}(\boldsymbol{\alpha})$, mixing per-objective LoRA adapters by $\boldsymbol{\lambda}^{(t)}$. The training loss at each step is $\sum_i\lambda_i^{(t)}\ell_i$, and its gradient on $\mathbb{R}^d$ is~\eqref{eq:appendix-linear-form}.

\paragraph{PaMaL~\citep{dimitriadis2023pamal} and FERERO~\citep{chen2024ferero}.}
Sweep $\boldsymbol{\lambda}^{(t)}$ within a single run to trace the Pareto front via preference-conditioned model surgery. The user-specified $\boldsymbol{\lambda}^{(t)}$ enters the update as the linear-scalarization coefficients of~\eqref{eq:appendix-linear-form}.

\subsection{Common Property: Reconciled Directions Are Euclidean}
\label{app:moo-euclidean}

Across all three families, $\boldsymbol{\lambda}^{(t)}$ is determined by a different mechanism (gradient geometry, loss history, or external preference), but the executed direction is the same object: a linear combination of per-objective Euclidean gradients with $\boldsymbol{\theta}$-detached coefficients, equal to the Euclidean gradient of $\ell_{\boldsymbol{\lambda}^{(t)}}$ on $\mathbb{R}^d$ by~\eqref{eq:appendix-euclidean-gradient}. This is the hypothesis under which Adam's diagonal RMS metric distorts the per-objective effect (Sec.~\ref{sec:mismatch}) and under which \method{}'s preference-conditioned curvature $\mathbf{C}_{\boldsymbol{\lambda}^{(t)}}$ in Eq.~\eqref{eq:scalarized-curvature} is the corresponding Fisher object.

\clearpage
\section{Proof of Prop.~\ref{prop:adam-geometry}}
\label{app:proof-prop-adam-geometry}

\begin{proof}
Recall from~\eqref{eq:adam-metric} that $\mathbf{M}_{\mathrm{Adam}}^{(t)} = \operatorname{Diag}(\sqrt{\hat{\mathbf{v}}^{(t)}}+\varepsilon)$ is the diagonal RMS metric: its inverse is the inverse-square-root-of-empirical-Fisher preconditioner Adam applies coordinate-wise.
By the first-order Taylor approximation,
\begin{equation}
    \Delta\ell_i^{(t)}
    \approx \mathbf{g}_i^{(t)\top}\Delta\boldsymbol{\theta}^{(t)}.
\end{equation}
From~\eqref{eq:adam_update} and~\eqref{eq:adam-metric}, the parameter step is
\begin{equation}
    \Delta\boldsymbol{\theta}^{(t)}
    = -\eta\,\frac{\hat{\mathbf{m}}^{(t)}}{\sqrt{\hat{\mathbf{v}}^{(t)}}+\varepsilon}
    = -\eta\,\bigl(\mathbf{M}_{\mathrm{Adam}}^{(t)}\bigr)^{-1}\hat{\mathbf{m}}^{(t)}.
\end{equation}
Substituting into the Taylor expansion gives~\eqref{eq:actual-effect}:
\begin{equation}
    \Delta\ell_i^{(t)}
    \approx \mathbf{g}_i^{(t)\top}\!\left(-\eta\,\bigl(\mathbf{M}_{\mathrm{Adam}}^{(t)}\bigr)^{-1}\hat{\mathbf{m}}^{(t)}\right)
    = -\eta\,\mathbf{g}_i^{(t)\top}\bigl(\mathbf{M}_{\mathrm{Adam}}^{(t)}\bigr)^{-1}\hat{\mathbf{m}}^{(t)}.
\end{equation}
The stationarity approximation $\hat{\mathbf{m}}^{(t)}\approx\mathbf{d}^{(t)}$ (stated in the proposition) is used in the subsequent analysis of \S\ref{sec:mismatch} and \S\ref{sec:q2}.
\end{proof}

\clearpage
\section{Sign Flips Under Adam's Geometric Mismatch}
\label{app:sign-flip}\label{app:sign-flip-geometric}

Sec.~\ref{sec:mismatch} introduces the geometric mismatch (Prop.~\ref{prop:adam-geometry}, Source~2): Adam's diagonal RMS metric $\mathbf{M}_{\mathrm{Adam}}^{(t)}$ distorts the bilinear form $\mathbf{g}_i^\top\mathbf{d}$, so that under the stationarity approximation $\hat{\mathbf{m}}^{(t)}\approx\mathbf{d}^{(t)}$ the per-objective loss change in Eq.~\eqref{eq:actual-effect} can flip from descent ($\Delta\ell_i^{(t)}<0$) to ascent on tasks the solver labeled aligned with $\mathbf{d}$. This appendix proves that mechanism: Lemma~\ref{lem:sign-flip} gives a tight necessary-and-sufficient condition under which a positive-definite diagonal metric flips the sign of $\mathbf{g}_i^\top\mathbf{M}^{-1}\mathbf{d}$ relative to the Euclidean reference $\mathbf{g}_i^\top\mathbf{d}$, expressed as a condition-number trigger $\rho(\mathbf{M})\ge P_i/N_i$ on the per-task alignment ratio. The remark following the lemma instantiates this trigger to Adam, showing that the per-coordinate anisotropy of $\mathbf{M}_{\mathrm{Adam}}^{(t)}$ in trained networks is large enough to satisfy the trigger even for tasks with substantial Euclidean alignment with $\mathbf{d}$.

Working under the stationarity approximation $\hat{\mathbf{m}}^{(t)}\approx\mathbf{d}^{(t)}$, the realized per-objective loss change in Prop.~\ref{prop:adam-geometry} reduces to $-\eta\,\mathbf{g}_i^{(t)\top}(\mathbf{M}_{\mathrm{Adam}}^{(t)})^{-1}\mathbf{d}^{(t)}$. The next lemma characterizes exactly when a positive-definite diagonal metric flips this bilinear form's sign relative to the Euclidean reference $\mathbf{g}_i^{(t)\top}\mathbf{d}^{(t)}$.

\begin{lemma}[Sign-Flip Condition]\label{lem:sign-flip}
Let $\mathbf{g}_i, \mathbf{d} \in \mathbb{R}^d$ and let $\mathbf{M} = \operatorname{Diag}(m_1,\ldots,m_d)$ be a positive-definite diagonal metric with condition number $\rho(\mathbf{M}) := m_{\max}/m_{\min}$.
Define the positive and negative parts of the coordinate-wise product:
\begin{equation}
    P_i := \sum_{\substack{k:\, g_{ik}d_k>0}} g_{ik}d_k,
    \qquad
    N_i := \Bigl|\sum_{\substack{k:\, g_{ik}d_k<0}} g_{ik}d_k\Bigr|,
\end{equation}
so that $\mathbf{g}_i^\top\mathbf{d} = P_i - N_i$.
Suppose $\mathbf{g}_i^\top\mathbf{d} > 0$ (task $i$ is aligned with $\mathbf{d}$ in the Euclidean metric) and $N_i > 0$.
There exists a diagonal $\mathbf{M}$ with condition number $\rho$ such that $\mathbf{g}_i^\top\mathbf{M}^{-1}\mathbf{d} \le 0$ if and only if
\begin{equation}\label{eq:sign-flip-condition}
    \rho \;\ge\; \frac{P_i}{N_i}.
\end{equation}
\end{lemma}

\begin{proof}
We have $\mathbf{g}_i^\top\mathbf{M}^{-1}\mathbf{d} = \sum_k g_{ik}d_k/m_k$.
Partition the coordinates as $\mathcal{I}^+ = \{k : g_{ik}d_k > 0\}$ and $\mathcal{I}^- = \{k : g_{ik}d_k < 0\}$.

\emph{($\Leftarrow$)}\ Suppose $\rho \ge P_i/N_i$.
Set $m_k = m_{\max}$ for $k \in \mathcal{I}^+$ and $m_k = m_{\min}$ for $k \in \mathcal{I}^-$.
Then
\[
    \mathbf{g}_i^\top\mathbf{M}^{-1}\mathbf{d}
    = \frac{P_i}{m_{\max}} - \frac{N_i}{m_{\min}}
    = \frac{1}{m_{\min}}\!\left(\frac{P_i}{\rho} - N_i\right)
    \le 0,
\]
since $P_i/\rho \le N_i$ by assumption.

\emph{($\Rightarrow$)}\ Suppose $\rho < P_i/N_i$.
For any diagonal $\mathbf{M}$ with condition number $\rho$,
\[
    \mathbf{g}_i^\top\mathbf{M}^{-1}\mathbf{d}
    \ge \frac{P_i}{m_{\max}} - \frac{N_i}{m_{\min}}
    = \frac{P_i - \rho\, N_i}{m_{\max}} > 0,
\]
so no sign flip is possible.
\end{proof}

\begin{remark}
The condition $\rho \ge P_i/N_i$ has a geometric interpretation.
The ratio $P_i/N_i$ measures how lopsided the alignment between $\mathbf{g}_i$ and $\mathbf{d}$ is: a large ratio means positive contributions dominate (task $i$ strongly agrees with the reconciled direction), so a proportionally large condition number is needed to flip the sign.
Conversely, when task $i$ is only marginally aligned ($P_i \approx N_i$, i.e.\ $\mathbf{g}_i^\top\mathbf{d} \approx 0$), even a mildly ill-conditioned metric can trigger a flip.
For unit-norm vectors with cosine similarity $c = \mathbf{g}_i^\top\mathbf{d}$, the minimum condition number enabling a sign flip satisfies
\begin{equation}
    \rho_{\min}(c)
    = \frac{P_i}{N_i}
    = \frac{1 + c}{1 - c},
    \qquad
    P_i = \frac{1+c}{2},\;N_i = \frac{1-c}{2},
\end{equation}
achieved when $P_i + N_i = 1$.
Notably, when $\mathbf{g}_i = \mathbf{d}$ (single-objective optimization), $c = 1$ and $\rho_{\min}(1) = \infty$: no finite condition number can flip the sign, consistent with $\mathbf{d}^\top\mathbf{M}^{-1}\mathbf{d} > 0$ for any positive-definite $\mathbf{M}$.
In MOO, $c < 1$ in general, so $\rho_{\min}(c)$ is finite.
In practice, Adam's per-coordinate effective scale $\sqrt{\hat{v}_k}$ is highly anisotropic and routinely spans several orders of magnitude in trained networks~\citep{tong2022calibrating,zhang2024transformers}, so the trigger $\rho \ge P_i/N_i$ is satisfied even for tasks with substantial alignment ($c$ close to~$1$).
\end{remark}

\clearpage
\section{Toy-Problem Visualizations}
\label{app:toy-visualizations}

This appendix collects the two-dimensional toy experiments that illustrate the weighting mismatch (Source~1) and the geometric mismatch (Source~2) of Sec.~\ref{sec:mismatch}.

\subsection{Multi-Phase Weighting Mismatch}
\label{app:multiphase-weighting}

This appendix expands the left panel of Fig.~\ref{fig:mismatch-side-by-side} into a multi-phase visualization of the weighting mismatch (Source~1, Sec.~\ref{sec:mismatch}). Whereas Fig.~\ref{fig:mismatch-side-by-side} (left) shows a single preference switch $\boldsymbol{\lambda}_A\!\to\!\boldsymbol{\lambda}_B$, here we drive the same 2D problem through six consecutive phases of alternating preferences $\boldsymbol{\lambda}_A\!\leftrightarrow\!\boldsymbol{\lambda}_B$, exposing how Adam's history-averaged second moment compounds across repeated switches while \method{}'s preference-conditioned preconditioner re-aligns the update geometry phase by phase.

\begin{figure}[!htbp]
    \centering
    \includegraphics[width=\linewidth]{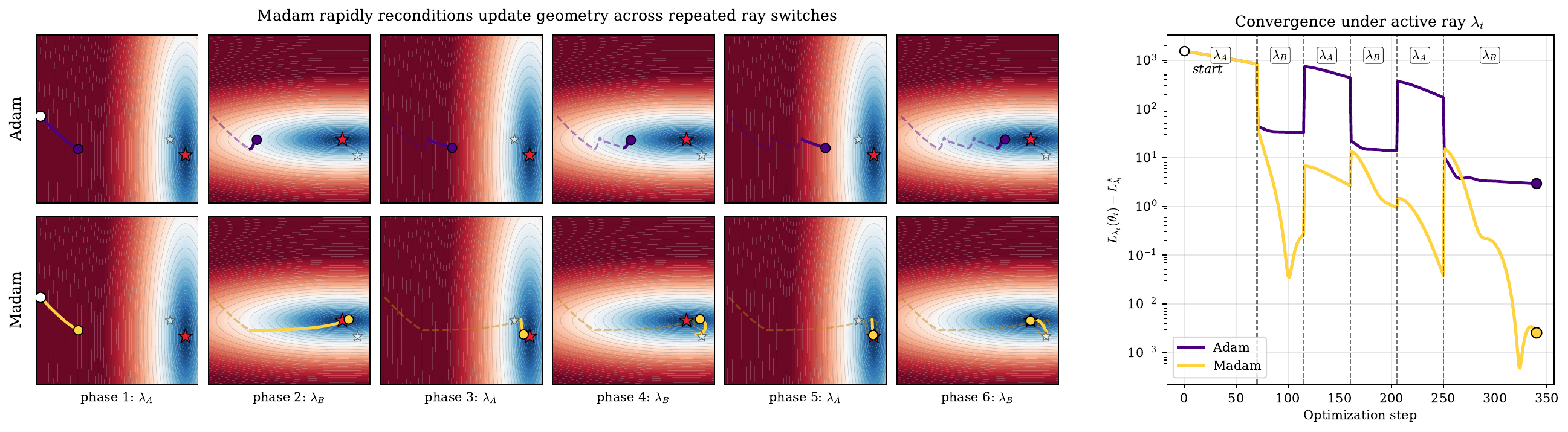}
    \caption{
    \textbf{Multi-phase weighting mismatch: repeated preference switches $\boldsymbol{\lambda}_A\!\leftrightarrow\!\boldsymbol{\lambda}_B$.}
    \emph{Left grid:} decision-space iterates of Adam (top row) and \method{} (bottom row) across six consecutive phases. Each column fixes the active preference $\boldsymbol{\lambda}^{(t)}\in\{\boldsymbol{\lambda}_A,\boldsymbol{\lambda}_B\}$; the background is the corresponding scalarized landscape $\ell_{\boldsymbol{\lambda}^{(t)}}$, the red star marks $\boldsymbol{\theta}^\star_{\boldsymbol{\lambda}^{(t)}}$, and the gray star marks the previous-phase optimum. Adam's iterate (purple) lags behind the active optimum after each switch and accumulates bias as phases alternate; \method{}'s iterate (yellow) re-converges to the active $\boldsymbol{\theta}^\star_{\boldsymbol{\lambda}^{(t)}}$ within each phase.
    \emph{Right:} scalarized loss gap $\ell_{\boldsymbol{\lambda}^{(t)}}(\boldsymbol{\theta}^{(t)})-\ell^\star_{\boldsymbol{\lambda}^{(t)}}$ on the active preference, plotted on a log scale across all six phases. Vertical dashed lines mark phase boundaries. Adam's gap spikes at every switch and plateaus several orders of magnitude above the active optimum, while \method{} drives the gap down in each phase.
    The contrast confirms Prop.~\ref{prop:adam-weighting-mismatch}: Adam's $\hat{\mathbf{v}}^{(t)}$ tracks the weight-marginalized curvature averaged across past preferences, so the realized metric never matches the active $\boldsymbol{\lambda}^{(t)}$; \method{}'s preference-conditioned $\mathbf{C}_{\boldsymbol{\lambda}^{(t)}}$ instantiates the right metric at each step.
    }
    \label{fig:multiphase-weighting}
\end{figure}

\subsection{Decision-Space Trajectories Under Rotating Curvature}
\label{app:madam-trajectory}

To make the geometric mismatch of Sec.~\ref{sec:mismatch} concrete at the optimizer level, we visualize \method{} and Adam trajectories on a two-dimensional toy multi-objective problem with a preference-rotating scalarized metric. The two objectives are matched-scale convex quadratics with opposite principal axes:
\begin{equation}
    f_1(\mathbf{x}) = \tfrac{1}{2}(\mathbf{x}-\mathbf{c}_1)^{\!\top} \mathbf{H}_1 (\mathbf{x}-\mathbf{c}_1),
    \quad \mathbf{c}_1 = (1.5,\,0),
\end{equation}
\begin{equation}
    f_2(\mathbf{x}) = \tfrac{1}{2}(\mathbf{x}-\mathbf{c}_2)^{\!\top} \mathbf{H}_2 (\mathbf{x}-\mathbf{c}_2),
    \quad \mathbf{c}_2 = (-1.5,\,0).
\end{equation}
With tilt parameters $\alpha = 2.0$ and $\beta = 0.2$, the Hessians are
\begin{equation}
    \mathbf{H}_1 = \begin{bmatrix} 1.1 & 0.9 \\ 0.9 & 1.1 \end{bmatrix},
    \qquad
    \mathbf{H}_2 = \begin{bmatrix} 1.1 & -0.9 \\ -0.9 & 1.1 \end{bmatrix},
\end{equation}
so that, written out coordinate-wise,
\begin{equation}
    f_1(x_1,x_2) = 0.55\,(x_1-1.5)^2 + 0.9\,(x_1-1.5)\,x_2 + 0.55\,x_2^2,
\end{equation}
\begin{equation}
    f_2(x_1,x_2) = 0.55\,(x_1+1.5)^2 - 0.9\,(x_1+1.5)\,x_2 + 0.55\,x_2^2.
\end{equation}

Both Hessians share the same eigenvalues, so neither task is simply ``larger-scale'' than the other; the difference is purely directional. The objective $f_1$ is stiff along the $(1,1)$ direction, while $f_2$ is stiff along $(1,-1)$. Consequently, as the preference vector $\boldsymbol{\lambda}=(\lambda_1,\lambda_2)$ varies, the scalarized objective $\lambda_1 f_1 + \lambda_2 f_2$ has its active curvature direction rotated, and the correct local metric rotates with it.

\begin{figure}[!htbp]
    \centering
    \includegraphics[width=\linewidth]{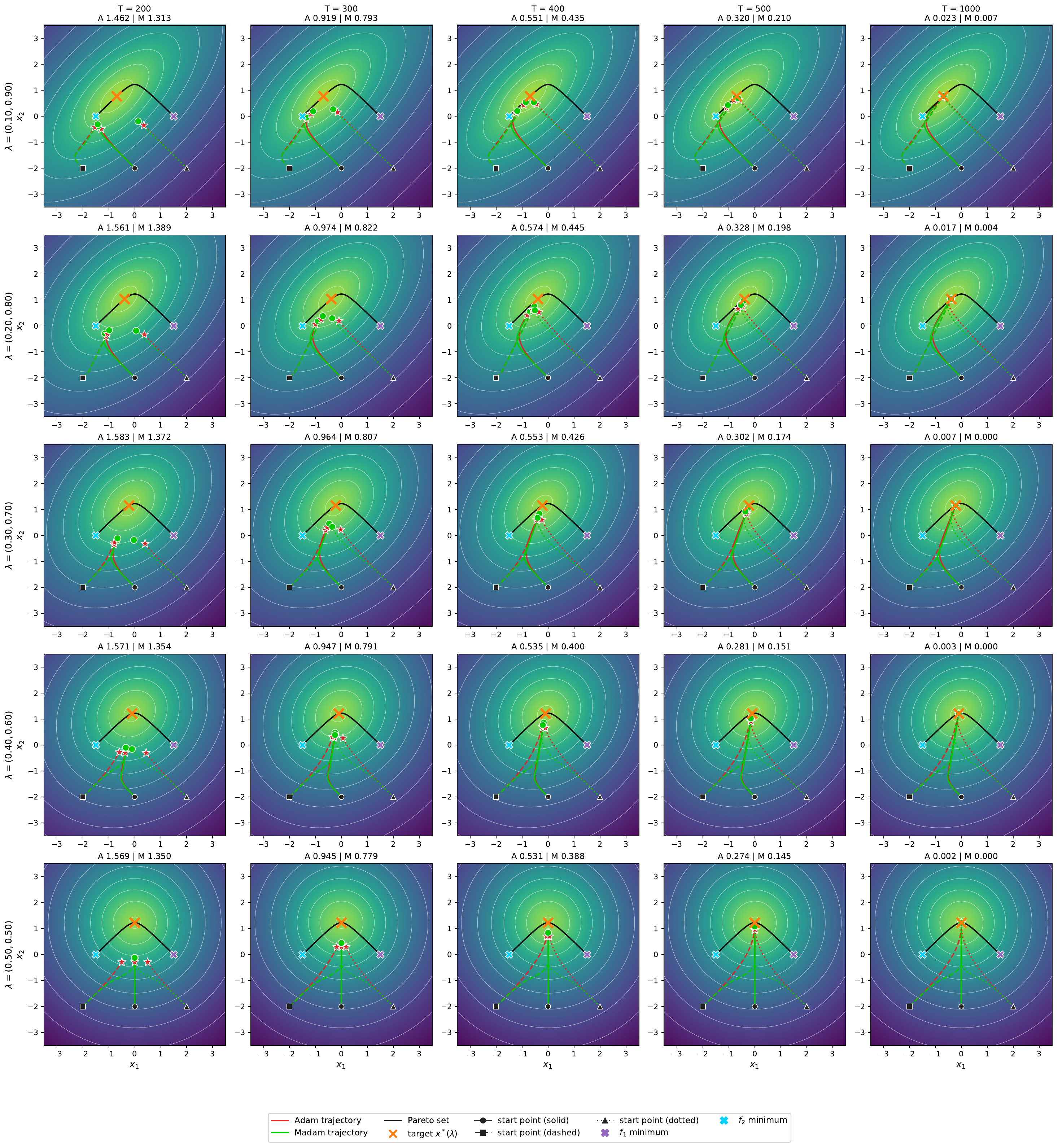}
    \caption{
    \textbf{Decision-space trajectories under changing scalarization geometry.}
    Each row fixes a preference vector $\boldsymbol{\lambda}=(\lambda_1,\lambda_2)$, and each column shows the optimizer state after a different number of steps. The background in each row is the corresponding scalarized objective $\lambda_1 f_1 + \lambda_2 f_2$ for the pair of matched-scale tilted quadratics defined above. The black curve shows the analytic Pareto set, and the orange cross marks the row-specific scalarized optimum $\mathbf{x}^\star(\boldsymbol{\lambda})$. Red and green trajectories compare Adam and \method{} from three shared initial points. The annotation ``A $d_{\mathrm{A}}$ $\mid$ M $d_{\mathrm{M}}$'' in each panel (e.g.\ \texttt{A 1.578 | M 1.363}) reports the mean Euclidean distance from the current iterate to the row-specific scalarized optimum $\mathbf{x}^\star(\boldsymbol{\lambda})$, averaged across the three initializations, for Adam (A) and \method{} (M) respectively. Because the two objectives have opposite principal-curvature directions, changing $\boldsymbol{\lambda}$ changes the correct local metric; \method{} more directly tracks this preference-conditioned geometry, while Adam uses coordinate-wise adaptive scaling in the ambient parameter basis.
    }
    \label{fig:madam-trajectory}
\end{figure}

The visualization mirrors the analysis of Sec.~\ref{sec:mismatch}: because the relevant curvature direction rotates with $\boldsymbol{\lambda}$, any coordinate-wise diagonal preconditioner is locked to the ambient basis and cannot follow this rotation, whereas \method{}'s preference-conditioned metric adapts to it. The contrast is visible in the trajectories—\method{} bends toward $\mathbf{x}^\star(\boldsymbol{\lambda})$ along the row-specific stiff direction, while Adam takes a path shaped by per-coordinate adaptive scales that do not depend on $\boldsymbol{\lambda}$.

\clearpage
\section{Experiment Setup: Datasets and Multi-Objectives}
\label{app:experiment-setup}

This appendix details every benchmark used in Sec.~\ref{sec:experiments}: the dataset, the model, the set of objectives $\{\ell_i\}_{i=1}^C$ that define the MOO problem, and the metrics reported. Train/validation/test splits and architectures follow the corresponding original references unless noted.

\subsection{Multi-Task Learning and Pareto MTL Benchmarks}
\label{app:setup-mtl}

We adopt the experimental protocol of PaLoRA~\citep{dimitriadis2025palora}. For each benchmark we evaluate two regimes: the standard MTL setting, in which a single fixed preference vector $\boldsymbol{\lambda}$ (taken to be uniform unless stated otherwise) is used throughout training, and the Pareto MTL setting, in which a family of preference rays is used to trace out the Pareto front.

\paragraph{$\mathtt{MultiMNIST}$~\citep{lecun1998mnist,sener2018multi}.}
Each $36\times36$ input is synthesized by superimposing two MNIST digits, one translated toward the top-left corner and the other toward the bottom-right.
\begin{itemize}\itemsep0pt
    \item \textbf{Tasks ($C=2$):} (i) 10-way classification of the top-left digit and (ii) 10-way classification of the bottom-right digit.
    \item \textbf{Losses:} per-objective cross-entropy.
    \item \textbf{Backbone:} a multi-head LeNet comprising a shared convolutional trunk and two classification heads.
    \item \textbf{Metrics:} per-objective top-1 accuracy.
    \item \textbf{Pareto setup:} hypervolume (HV) computed over $11$ uniformly spaced preference rays on the $2$-simplex, namely $(0.0, 1.0), (0.1, 0.9), \dots, (1.0, 0.0)$, in the $2$-D accuracy space. Reference point: $[0, 0]$.
\end{itemize}

\paragraph{$\mathtt{SARCOS}$~\citep{vijayakumar2000locally}.}
A canonical inverse-dynamics regression benchmark for the seven degrees of freedom of the $\mathtt{SARCOS}$ robot arm: given a $21$-dimensional state vector consisting of joint positions, velocities, and accelerations, the task is to predict the seven corresponding joint torques.
\begin{itemize}\itemsep0pt
    \item \textbf{Tasks ($C=7$):} one regression task per joint torque.
    \item \textbf{Losses:} per-objective mean-squared error.
    \item \textbf{Backbone:} a shared MLP trunk with seven task-specific linear regression heads.
    \item \textbf{Metrics:} per-objective MSE (scaled by a factor of $100$ for readability).
    \item \textbf{Pareto setup:} HV computed over $8$ preference rays in the $7$-D error space: the seven canonical one-hot rays $\mathbf{e}_1,\dots,\mathbf{e}_7$ together with the uniform ray $(1/7,\dots,1/7)$. Reference point: $[1, 5, 1, 1, 10, 10, 1]$.
\end{itemize}

\paragraph{$\mathtt{UTKFace}$~\citep{zhang2017age}.}
A collection of aligned and cropped facial images annotated with age, gender, and race.
\begin{itemize}\itemsep0pt
    \item \textbf{Tasks ($C=3$):} (i) age regression (continuous, normalized to $[0,1]$), (ii) binary gender classification, and (iii) $5$-way race classification.
    \item \textbf{Losses:} Huber loss for age, binary cross-entropy for gender, and cross-entropy for race.
    \item \textbf{Backbone:} a shared ResNet-18 trunk with three task-specific heads.
    \item \textbf{Metrics:} Huber loss for age (scaled by a factor of $100$ for readability) and top-1 accuracy for gender and race.
    \item \textbf{Pareto setup:} the same trained model is queried at each preference vector $\boldsymbol{\lambda}=(\lambda_{\mathrm{age}},\lambda_{\mathrm{gender}},\lambda_{\mathrm{race}})$ drawn from a uniform triangulation of the $3$-simplex of order $10$, yielding $66$ preference rays, and HV is computed in the $3$-D metric space. Reference point: $[0.5, 0.0, 0.0]$.
\end{itemize}

\paragraph{$\mathtt{Cityscapes}$~\citep{cordts2016cityscapes}.}
A large-scale urban driving dataset providing dense pixel-level annotations. Following standard practice in the MTL literature, we adopt the downsampled $128\times256$ resolution variant.
\begin{itemize}\itemsep0pt
    \item \textbf{Tasks ($C=2$):} (i) $19$-class semantic segmentation and (ii) per-pixel inverse-depth regression.
    \item \textbf{Losses:} pixel-wise cross-entropy for segmentation and an $L_1$ regression loss for depth.
    \item \textbf{Backbone:} a SegNet-style encoder--decoder architecture with a shared encoder and two task-specific decoders.
    \item \textbf{Metrics:} mean intersection-over-union (mIoU) and pixel accuracy for segmentation; mean absolute and mean relative error for depth.
    \item \textbf{Pareto setup:} the same trained model is queried at each of the $11$ uniformly spaced preference rays $(0.0, 1.0), (0.1, 0.9), \dots, (1.0, 0.0)$ on the $2$-simplex, and the best ray under the metric of interest is selected for comparison.
\end{itemize}

\paragraph{$\mathtt{NYUv2}$~\citep{silberman2012indoor}.}
Indoor RGB-D scenes with three dense annotation types.
\begin{itemize}\itemsep0pt
    \item \textbf{Tasks ($C=3$):} (i) $13$-class semantic segmentation, (ii) per-pixel depth regression, and (iii) surface-normal prediction.
    \item \textbf{Losses:} pixel-wise cross-entropy for segmentation, an $L_1$ loss for depth, and a cosine distance $1-\cos\theta$ for surface normals.
    \item \textbf{Backbone:} a SegNet-style encoder--decoder architecture with a shared encoder and three task-specific decoders.
    \item \textbf{Metrics:} mIoU and pixel accuracy for segmentation; mean absolute and mean relative error for depth; mean and median angular error together with the within-threshold accuracies at $\{11.25^\circ, 22.5^\circ, 30^\circ\}$ for normals.
    \item \textbf{Pareto setup:} the same trained model is queried at each preference vector $\boldsymbol{\lambda}=(\lambda_{\mathrm{seg}},\lambda_{\mathrm{depth}},\lambda_{\mathrm{normal}})$ drawn from a uniform triangulation of the $3$-simplex of order $10$, yielding $66$ preference rays, and the best ray under the metric of interest is selected for comparison.
\end{itemize}

\subsection{Physics-Informed Neural Networks (PINNs)}
\label{app:setup-pinn}

We use the PINNacle benchmark~\citep{hao2024pinnacle} with $20$ PDE cases spanning Burgers, heat, Poisson, wave, Navier--Stokes, chaotic systems, and high-dimensional Poisson and heat families. For each PDE, the neural network $u_{\boldsymbol{\theta}}$ is trained to satisfy the governing equation in the interior together with the problem-specific boundary, initial, periodic, and observation constraints. Each PDE therefore induces a multi-objective problem whose objectives are the constituent residual and constraint losses, all naturally on different and time-varying scales.

\begin{itemize}\itemsep0pt
    \item \textbf{Tasks (per PDE):}
    Each task corresponds to one residual or constraint loss. Depending on the PDE, this includes one or more PDE residual losses $\ell_{\mathrm{PDE}}$ on interior collocation points, boundary-condition losses $\ell_{\mathrm{BC}}$, initial-condition losses $\ell_{\mathrm{IC}}$, periodicity losses, and point-observation constraint losses. The number of task losses is PDE-dependent and ranges from $2$ to $11$ across the active $20$ PINNacle cases.
    \item \textbf{Losses:} squared residual or constraint errors evaluated at the corresponding interior, boundary, initial, periodic, or point-observation samples.
    \item \textbf{Backbone:} fully connected MLP $u_{\boldsymbol{\theta}}$ with $\tanh$ activations and Glorot-normal initialization; input and output dimensions are determined by the PDE.
    \item \textbf{Metrics:} relative $L_2$ error against the available reference solution or reference data, together with final per-component training and test losses.
\end{itemize}

\subsection{Medical Image Analysis}
\label{app:setup-medical}

We evaluate \method{} on two medical-imaging benchmarks in which the multi-objective formulation arises from \emph{loss decomposition} rather than from multiple supervised tasks: skin-lesion segmentation on $\mathtt{ISIC2018}$ and brain-MRI super-resolution on $\mathtt{OASIS3}$. In both settings the constituent losses sit on markedly different scales and have qualitatively different gradient distributions, exposing the first-moment lag analyzed in Sec.~\ref{sec:mismatch} while remaining clean low-dimensional MOO testbeds.

\paragraph{$\mathtt{ISIC2018}$~\citep{codella2019skin}.}
The ISIC2018 Task~1 lesion-boundary segmentation challenge consists of dermoscopic images paired with binary lesion masks, with a training set of $\sim\!2{,}156$ images, a validation set of $\sim\!538$ images and a test set of $1000$ images. We follow the standard preprocessing pipeline (resize to a fixed input resolution, intensity normalization).
\begin{itemize}\itemsep0pt
    \item \textbf{Tasks ($C=2$):} a single binary lesion-segmentation prediction supervised by two complementary objectives, a pixelwise classification loss and a region-overlap loss.
    \item \textbf{Losses:} pixelwise cross-entropy
    and soft-Dice loss.
    Cross-entropy is dominated by easy interior pixels while Dice is dominated by boundary errors, yielding gradients on substantially different scales.
    \item \textbf{Backbones:} UNet~\citep{ronneberger2015u} and Swin-Transformer~\citep{liu2021swin}, trained from scratch.
    \item \textbf{Metrics:} validation Dice coefficient.
\end{itemize}

\paragraph{$\mathtt{OASIS3}$~\citep{lamontagne2019oasis}.}
A brain-MRI super-resolution benchmark in which a $4\times$ downsampled image is reconstructed to its original resolution. We use a subject-level split of $1{,}129$ volumes for training and $281$ for held-out validation. Intensities are clipped to the $[p_1, p_{99}]$ range and min--max normalized to $[0,1]$; the low-resolution input is generated on the fly by isotropic $4\times$ trilinear downsampling followed by trilinear upsampling back to the original grid.
\begin{itemize}\itemsep0pt
    \item \textbf{Tasks ($C=5$):} a single high-resolution reconstruction supervised by one pixelwise loss and four auxiliary losses chosen to capture high-frequency detail and structural fidelity.
    \item \textbf{Losses:} L1 pixelwise reconstruction, Sobel edge, FFT spectral, Laplacian high-pass, and SSIM structural-similarity. The five terms differ in scale by more than an order of magnitude and exhibit distinct gradient distributions.
    \item \textbf{Backbone:} SwinUNETR~\citep{hatamizadeh2021swin} trained from scratch.
    \item \textbf{Metrics:} LPIPS, DISTS, PSNR and SSIM computed per subject.
\end{itemize}

\clearpage
\section{Additional Experimental Results}
\label{app:additional-results}

This appendix collects results that supplement Sec.~\ref{sec:experiments}: the $\mathtt{NYUv2}$ MTL table, per-task breakdowns of the Pareto MTL results, per-PDE PINN tables, and the two additional comparisons requested during review (recent MTL solvers and curvature-aware optimizers).

\subsection{$\mathtt{NYUv2}$ MTL Results}\label{app:nyuv2-mtl}

Tab.~\ref{tab:nyuv2-mtl} reports the full $\mathtt{NYUv2}$ MTL results referenced in the main text.

\begin{table}[!htbp]
\centering
\scriptsize
\setlength{\tabcolsep}{3pt}
\caption{$\mathtt{NYUv2}$ MTL results (mean over three seeds). Surface-normal columns report angle errors (Mean, Med.) and the percentage of pixels within $11.25^\circ\!/22.5^\circ\!/30^\circ$. Bold marks the better value within each \mbox{Base/+\method{}} block.}
\label{tab:nyuv2-mtl}
\resizebox{\linewidth}{!}{%
\begin{tabular}{lccccccccc}
\toprule
& \multicolumn{2}{c}{Seg.} & \multicolumn{2}{c}{Depth} & \multicolumn{5}{c}{Surface Normal} \\
\cmidrule(lr){2-3}\cmidrule(lr){4-5}\cmidrule(lr){6-10}
Method & mIoU $\uparrow$ & Pix Acc $\uparrow$ & Abs Err $\downarrow$ & Rel Err $\downarrow$ & Mean $\downarrow$ & Med. $\downarrow$ & $<\!11.25^\circ$ $\uparrow$ & $<\!22.5^\circ$ $\uparrow$ & $<\!30^\circ$ $\uparrow$ \\
\midrule
LS       & 36.33 & 63.57 & \textbf{0.544} & 0.222 & 27.54 & 21.69 & 26.59 & 52.00 & 64.24 \\
\rlightgray
+ Madam  & \textbf{37.56} & \textbf{64.13} & 0.550 & \textbf{0.221} & \textbf{27.31} & \textbf{21.60} & \textbf{26.64} & \textbf{52.60} & \textbf{64.86} \\
\midrule
DWA      & 37.07 & 63.55 & \textbf{0.541} & 0.220 & \textbf{27.63} & \textbf{21.90} & 26.40 & 51.59 & 63.91 \\
\rlightgray
+ Madam  & \textbf{37.21} & \textbf{63.89} & \textbf{0.541} & \textbf{0.218} & \textbf{27.63} & 22.05 & \textbf{26.41} & \textbf{51.71} & \textbf{64.49} \\
\midrule
UW       & \textbf{32.60} & 60.87 & \textbf{0.540} & 0.220 & \textbf{28.22} & \textbf{22.43} & \textbf{26.05} & \textbf{50.60} & \textbf{62.69} \\
\rlightgray
+ Madam  & 32.59 & \textbf{60.91} & 0.546 & \textbf{0.217} & 28.29 & 22.67 & 25.55 & 50.18 & 62.45 \\
\midrule
IMTL     & 36.14 & 63.51 & \textbf{0.536} & \textbf{0.213} & \textbf{25.90} & 19.74 & 29.74 & 55.98 & 67.80 \\
\rlightgray
+ Madam  & \textbf{36.26} & \textbf{63.68} & 0.543 & 0.214 & 26.04 & \textbf{19.70} & \textbf{29.84} & \textbf{56.03} & \textbf{67.81} \\
\midrule
CAGrad   & \textbf{36.99} & \textbf{63.90} & \textbf{0.543} & \textbf{0.216} & 25.87 & 19.72 & \textbf{29.58} & 56.12 & 67.96 \\
\rlightgray
+ Madam  & 36.88 & 63.64 & 0.546 & 0.224 & \textbf{25.56} & \textbf{19.55} & 29.35 & \textbf{56.32} & \textbf{68.39} \\
\midrule
PaMaL    & \textbf{33.94} & \textbf{62.55} & 0.5592 & 0.2188 & 26.60 & 20.33 & 29.09 & 54.61 & 66.35 \\
\rlightgray
+ Madam  & 33.71 & 62.34 & \textbf{0.5540} & \textbf{0.2165} & \textbf{26.47} & \textbf{20.24} & \textbf{29.13} & \textbf{54.70} & \textbf{66.42} \\
\midrule
PaLoRA   & 38.27 & 64.79 & \textbf{0.5370} & \textbf{0.2150} & 25.66 & 19.34 & 30.47 & 56.90 & 68.56 \\
\rlightgray
+ Madam  & \textbf{38.32} & \textbf{64.89} & 0.5398 & 0.2167 & \textbf{25.60} & \textbf{19.28} & \textbf{30.57} & \textbf{57.07} & \textbf{68.75} \\
\bottomrule
\end{tabular}
}
\end{table}

\subsection{Per-Task Pareto MTL Results}\label{app:multimnist-pareto-per-objective}

Tabs.~\ref{tab:multimnist-pareto-mtl-per-objective}--\ref{tab:sarcos-pareto-mtl-palora-per-objective} break the Pareto MTL results of Sec.~\ref{sec:experiments} down by objective and preference ray (mean $\pm$ std over three seeds).

\subsubsection{\texorpdfstring{$\mathtt{MultiMNIST}$}{MultiMNIST}}

Tab.~\ref{tab:multimnist-pareto-mtl-per-objective} expands the $\mathtt{MultiMNIST}$ Pareto fronts of Fig.~\ref{fig:multimnist-pareto-mtl} to a per-objective breakdown for both Pareto-front solvers. For each preference ray $r_0,\dots,r_{10}$ we list the top-left (TL) and bottom-right (BR) digit accuracies separately, together with the $2$-D hypervolume HV computed in the two-task accuracy space against the reference point $[0,0]$. \method{} raises HV from $0.9016$ to $0.9177$ for PaLoRA and from $0.8869$ to $0.9086$ for PaMaL, and improves both per-objective accuracies on every ray for PaLoRA. For PaMaL the gain is uniform on the central rays $r_2,\dots,r_8$ but reverses on the extreme rays $r_9, r_{10}$ on the TL task and on $r_0$ on the BR task: at these endpoints the base solver pushes harder on the dominant task at the expense of the other, whereas \method{} produces a more balanced trade-off, which is consistent with the front-shape change visible in Fig.~\ref{fig:multimnist-pareto-mtl}.

\begin{table}[!htbp]
\centering
\scriptsize
\setlength{\tabcolsep}{2.5pt}
\caption{$\mathtt{MultiMNIST}$ Pareto MTL results, per-objective breakdown (mean $\pm$ std over three seeds). For each base solver and each \mbox{Base/+\method{}} variant we list the top-left (TL) and bottom-right (BR) digit accuracies at every preference ray $r_0,\dots,r_{10}$, together with HV in the two-task accuracy space.}
\label{tab:multimnist-pareto-mtl-per-objective}
\resizebox{\linewidth}{!}{%
\begin{tabular}{llcccccccccccc}
\toprule
Method & Task & HV $\uparrow$ & $r_0$ & $r_1$ & $r_2$ & $r_3$ & $r_4$ & $r_5$ & $r_6$ & $r_7$ & $r_8$ & $r_9$ & $r_{10}$ \\
\midrule
\multirow{2}{*}{PaMaL} & TL & \multirow{2}{*}{0.8869 $\pm$ 0.0060} & 94.73 $\pm$ 0.05 & 94.54 $\pm$ 0.12 & 94.21 $\pm$ 0.10 & 93.96 $\pm$ 0.03 & 93.75 $\pm$ 0.05 & 93.55 $\pm$ 0.11 & 93.00 $\pm$ 0.24 & 91.89 $\pm$ 0.35 & 89.48 $\pm$ 0.27 & \textbf{85.12 $\pm$ 0.82} & \textbf{77.44 $\pm$ 1.76} \\
 & BR &  & \textbf{76.44 $\pm$ 1.81} & 83.97 $\pm$ 1.14 & 88.33 $\pm$ 0.46 & 90.68 $\pm$ 0.38 & 91.88 $\pm$ 0.49 & 92.39 $\pm$ 0.57 & 92.57 $\pm$ 0.68 & 92.80 $\pm$ 0.78 & 93.14 $\pm$ 0.82 & 93.50 $\pm$ 0.68 & 93.83 $\pm$ 0.50 \\
\rlightgray
 & TL &  & \textbf{96.01 $\pm$ 0.21} & \textbf{95.74 $\pm$ 0.16} & \textbf{95.48 $\pm$ 0.24} & \textbf{95.30 $\pm$ 0.27} & \textbf{95.11 $\pm$ 0.18} & \textbf{94.81 $\pm$ 0.12} & \textbf{94.16 $\pm$ 0.14} & \textbf{92.86 $\pm$ 0.19} & \textbf{89.95 $\pm$ 0.22} & 83.63 $\pm$ 1.58 & 72.87 $\pm$ 4.34 \\
\rlightgray
\multirow{-2}{*}{+ \method{}} & BR & \multirow{-2}{*}{\textbf{0.9086 $\pm$ 0.0050}} & 74.32 $\pm$ 2.40 & 83.87 $\pm$ 1.38 & \textbf{89.15 $\pm$ 0.36} & \textbf{91.81 $\pm$ 0.12} & \textbf{93.13 $\pm$ 0.36} & \textbf{93.70 $\pm$ 0.39} & \textbf{94.03 $\pm$ 0.32} & \textbf{94.27 $\pm$ 0.39} & \textbf{94.51 $\pm$ 0.37} & \textbf{94.72 $\pm$ 0.32} & \textbf{94.80 $\pm$ 0.24} \\
\midrule
\multirow{2}{*}{PaLoRA} & TL & \multirow{2}{*}{0.9016 $\pm$ 0.0045} & 95.61 $\pm$ 0.08 & 95.63 $\pm$ 0.10 & 95.60 $\pm$ 0.11 & 95.58 $\pm$ 0.10 & 95.48 $\pm$ 0.11 & 95.32 $\pm$ 0.12 & 95.04 $\pm$ 0.14 & 94.71 $\pm$ 0.21 & 94.27 $\pm$ 0.24 & 93.85 $\pm$ 0.27 & 93.20 $\pm$ 0.38 \\
 & BR &  & 92.02 $\pm$ 0.94 & 92.55 $\pm$ 0.76 & 92.99 $\pm$ 0.67 & 93.26 $\pm$ 0.57 & 93.59 $\pm$ 0.45 & 93.83 $\pm$ 0.40 & 93.99 $\pm$ 0.38 & 94.19 $\pm$ 0.35 & 94.24 $\pm$ 0.28 & 94.22 $\pm$ 0.28 & 94.26 $\pm$ 0.27 \\
\rlightgray
 & TL &  & \textbf{96.45 $\pm$ 0.00} & \textbf{96.45 $\pm$ 0.05} & \textbf{96.36 $\pm$ 0.11} & \textbf{96.27 $\pm$ 0.11} & \textbf{96.20 $\pm$ 0.05} & \textbf{96.05 $\pm$ 0.08} & \textbf{95.82 $\pm$ 0.03} & \textbf{95.57 $\pm$ 0.03} & \textbf{95.11 $\pm$ 0.14} & \textbf{94.61 $\pm$ 0.18} & \textbf{93.93 $\pm$ 0.40} \\
\rlightgray
\multirow{-2}{*}{+ \method{}} & BR & \multirow{-2}{*}{\textbf{0.9177 $\pm$ 0.0024}} & \textbf{92.71 $\pm$ 0.46} & \textbf{93.30 $\pm$ 0.41} & \textbf{93.76 $\pm$ 0.35} & \textbf{94.06 $\pm$ 0.34} & \textbf{94.36 $\pm$ 0.35} & \textbf{94.55 $\pm$ 0.34} & \textbf{94.72 $\pm$ 0.26} & \textbf{94.90 $\pm$ 0.27} & \textbf{94.98 $\pm$ 0.27} & \textbf{95.06 $\pm$ 0.21} & \textbf{95.12 $\pm$ 0.17} \\
\bottomrule
\end{tabular}
}
\end{table}

\subsubsection{\texorpdfstring{$\mathtt{SARCOS}$}{SARCOS}}

Tabs.~\ref{tab:sarcos-pareto-mtl-pamal-per-objective} and~\ref{tab:sarcos-pareto-mtl-palora-per-objective} expand the $\mathtt{SARCOS}$ Pareto MTL results of Tab.~\ref{tab:sarcos-pareto-mtl} to a per-objective breakdown. For each preference ray $r_0,\dots,r_7$ we list the error of every joint task $T_1,\dots,T_7$ (scaled by $100$, mean~$\pm$~std across three seeds) together with the $7$-D hypervolume HV (computed against the reference point $[1,5,1,1,10,10,1]$). PaMaL is markedly less stable than PaLoRA, with occasional seeds blowing up on $T_2$ and $T_5$ at extreme rays; the inflated standard deviations in those cells reflect this rather than a systematic regression of \method{}.

\begin{table}[!htbp]
\centering
\scriptsize
\setlength{\tabcolsep}{2.5pt}
\caption{$\mathtt{SARCOS}$ Pareto MTL results for \textbf{PaMaL}, per-objective breakdown (errors $\times 100$, mean $\pm$ std over three seeds). HV is computed in the $7$-D error space. Bold marks the better value within each \mbox{Base/+\method{}} block (lower is better).}
\label{tab:sarcos-pareto-mtl-pamal-per-objective}
\resizebox{\linewidth}{!}{%
\begin{tabular}{llcrrrrrrrr}
\toprule
Method & Task & HV $\uparrow$ & $r_0$ $\downarrow$ & $r_1$ $\downarrow$ & $r_2$ $\downarrow$ & $r_3$ $\downarrow$ & $r_4$ $\downarrow$ & $r_5$ $\downarrow$ & $r_6$ $\downarrow$ & $r_7$ $\downarrow$ \\
\midrule
\multirow{7}{*}{PaMaL} & T1 & \multirow{7}{*}{378.29 $\pm$ 9.27} & \textbf{2.55 $\pm$ 0.13} & 8.15 $\pm$ 2.21 & \textbf{6.88 $\pm$ 2.15} & 6.24 $\pm$ 0.93 & 7.63 $\pm$ 1.07 & \textbf{7.08 $\pm$ 1.34} & 6.69 $\pm$ 1.67 & 5.72 $\pm$ 0.49 \\
 & T2 &  & 99.65 $\pm$ 31.18 & 35.68 $\pm$ 2.17 & \textbf{60.41 $\pm$ 4.89} & 80.37 $\pm$ 33.37 & 248.11 $\pm$ 194.08 & 63.79 $\pm$ 8.22 & 73.86 $\pm$ 5.44 & 46.55 $\pm$ 1.69 \\
 & T3 &  & 4.50 $\pm$ 1.53 & 3.03 $\pm$ 0.36 & \textbf{1.19 $\pm$ 0.07} & 2.67 $\pm$ 0.39 & 5.49 $\pm$ 1.45 & 2.51 $\pm$ 0.15 & 3.13 $\pm$ 0.70 & 1.79 $\pm$ 0.04 \\
 & T4 &  & 3.04 $\pm$ 0.53 & 2.69 $\pm$ 0.59 & \textbf{1.97 $\pm$ 0.41} & 0.57 $\pm$ 0.26 & 5.11 $\pm$ 3.03 & 1.89 $\pm$ 0.35 & 2.44 $\pm$ 1.02 & 1.35 $\pm$ 0.06 \\
 & T5 &  & 470.92 $\pm$ 468.44 & 197.50 $\pm$ 46.96 & \textbf{142.53 $\pm$ 13.23} & \textbf{139.74 $\pm$ 6.40} & 117.39 $\pm$ 26.31 & 181.94 $\pm$ 37.38 & 230.66 $\pm$ 75.37 & 114.95 $\pm$ 2.87 \\
 & T6 &  & \textbf{32.20 $\pm$ 2.91} & 39.14 $\pm$ 2.98 & \textbf{29.49 $\pm$ 1.90} & 27.79 $\pm$ 1.94 & 44.53 $\pm$ 5.93 & 22.75 $\pm$ 7.90 & 28.77 $\pm$ 1.19 & 25.75 $\pm$ 0.72 \\
 & T7 &  & 2.13 $\pm$ 0.22 & \textbf{2.35 $\pm$ 0.53} & \textbf{1.90 $\pm$ 0.55} & \textbf{1.63 $\pm$ 0.09} & 3.08 $\pm$ 1.48 & 3.01 $\pm$ 1.31 & 41.40 $\pm$ 57.86 & 1.24 $\pm$ 0.04 \\
\midrule
\rlightgray
 & T1 &  & 3.20 $\pm$ 1.51 & \textbf{6.60 $\pm$ 0.70} & 9.03 $\pm$ 5.98 & \textbf{5.23 $\pm$ 0.58} & \textbf{6.95 $\pm$ 1.43} & 7.53 $\pm$ 0.42 & \textbf{4.96 $\pm$ 0.39} & \textbf{4.95 $\pm$ 0.06} \\
\rlightgray
 & T2 &  & \textbf{65.91 $\pm$ 18.98} & \textbf{29.59 $\pm$ 0.23} & 236.04 $\pm$ 251.90 & \textbf{68.36 $\pm$ 20.81} & \textbf{86.71 $\pm$ 5.40} & \textbf{55.37 $\pm$ 5.25} & \textbf{59.14 $\pm$ 4.04} & \textbf{40.47 $\pm$ 3.63} \\
\rlightgray
 & T3 &  & \textbf{2.59 $\pm$ 0.42} & \textbf{2.91 $\pm$ 0.15} & 2.27 $\pm$ 1.63 & \textbf{2.29 $\pm$ 0.26} & \textbf{3.79 $\pm$ 1.56} & \textbf{2.29 $\pm$ 0.14} & \textbf{2.25 $\pm$ 0.17} & \textbf{1.63 $\pm$ 0.10} \\
\rlightgray
 & T4 &  & \textbf{2.04 $\pm$ 0.80} & \textbf{2.21 $\pm$ 0.33} & 4.84 $\pm$ 3.38 & \textbf{0.29 $\pm$ 0.02} & \textbf{2.72 $\pm$ 0.67} & \textbf{1.73 $\pm$ 0.14} & \textbf{1.82 $\pm$ 0.99} & \textbf{1.09 $\pm$ 0.05} \\
\rlightgray
 & T5 &  & \textbf{115.65 $\pm$ 5.26} & \textbf{152.62 $\pm$ 9.57} & 312.65 $\pm$ 240.60 & 144.33 $\pm$ 10.84 & \textbf{90.34 $\pm$ 1.34} & \textbf{128.65 $\pm$ 9.42} & \textbf{131.23 $\pm$ 15.43} & \textbf{104.75 $\pm$ 3.07} \\
\rlightgray
 & T6 &  & 32.97 $\pm$ 3.82 & \textbf{34.65 $\pm$ 2.61} & 33.92 $\pm$ 8.04 & \textbf{24.98 $\pm$ 1.81} & \textbf{37.56 $\pm$ 2.52} & \textbf{16.06 $\pm$ 2.13} & \textbf{28.15 $\pm$ 1.50} & \textbf{21.17 $\pm$ 0.29} \\
\rlightgray
\multirow{-7}{*}{+ \method{}} & T7 & \multirow{-7}{*}{\textbf{394.55 $\pm$ 4.20}} & \textbf{2.05 $\pm$ 0.33} & 2.78 $\pm$ 0.54 & 4.29 $\pm$ 2.91 & 1.70 $\pm$ 0.41 & \textbf{2.38 $\pm$ 0.27} & \textbf{2.48 $\pm$ 0.97} & \textbf{1.45 $\pm$ 0.78} & \textbf{1.00 $\pm$ 0.05} \\
\bottomrule
\end{tabular}
}
\end{table}

\begin{table}[!htbp]
\centering
\scriptsize
\setlength{\tabcolsep}{2.5pt}
\caption{$\mathtt{SARCOS}$ Pareto MTL results for \textbf{PaLoRA}, per-objective breakdown (errors $\times 100$, mean $\pm$ std over three seeds). HV is computed in the $7$-D error space. Bold marks the better value within each \mbox{Base/+\method{}} block (lower is better; ties to two decimals are left unbolded).}
\label{tab:sarcos-pareto-mtl-palora-per-objective}
\resizebox{\linewidth}{!}{%
\begin{tabular}{llcrrrrrrrr}
\toprule
Method & Task & HV $\uparrow$ & $r_0$ $\downarrow$ & $r_1$ $\downarrow$ & $r_2$ $\downarrow$ & $r_3$ $\downarrow$ & $r_4$ $\downarrow$ & $r_5$ $\downarrow$ & $r_6$ $\downarrow$ & $r_7$ $\downarrow$ \\
\midrule
\multirow{7}{*}{PaLoRA} & T1 & \multirow{7}{*}{431.21 $\pm$ 8.47} & 2.00 $\pm$ 0.14 & 3.15 $\pm$ 1.19 & 2.90 $\pm$ 1.18 & 2.65 $\pm$ 0.64 & 2.53 $\pm$ 0.34 & 2.27 $\pm$ 0.26 & 2.52 $\pm$ 0.66 & 2.16 $\pm$ 0.21 \\
 & T2 &  & 30.73 $\pm$ 9.03 & 22.06 $\pm$ 7.79 & 34.19 $\pm$ 19.67 & 23.56 $\pm$ 2.36 & 29.87 $\pm$ 5.78 & 32.99 $\pm$ 10.61 & 30.93 $\pm$ 9.54 & 21.17 $\pm$ 3.20 \\
 & T3 &  & 1.09 $\pm$ 0.32 & \textbf{1.32 $\pm$ 0.46} & \textbf{0.92 $\pm$ 0.22} & \textbf{0.95 $\pm$ 0.12} & \textbf{1.06 $\pm$ 0.07} & \textbf{0.93 $\pm$ 0.14} & \textbf{1.13 $\pm$ 0.42} & \textbf{0.94 $\pm$ 0.18} \\
 & T4 &  & 0.43 $\pm$ 0.18 & \textbf{0.36 $\pm$ 0.03} & 0.34 $\pm$ 0.07 & 0.30 $\pm$ 0.04 & 0.34 $\pm$ 0.04 & 0.39 $\pm$ 0.11 & 0.37 $\pm$ 0.10 & 0.31 $\pm$ 0.05 \\
 & T5 &  & 67.43 $\pm$ 7.00 & 69.84 $\pm$ 4.36 & 73.76 $\pm$ 17.93 & 69.35 $\pm$ 6.04 & 62.33 $\pm$ 5.38 & 84.22 $\pm$ 32.89 & 82.23 $\pm$ 21.24 & 60.15 $\pm$ 2.55 \\
 & T6 &  & 13.47 $\pm$ 2.60 & 14.03 $\pm$ 1.88 & \textbf{15.55 $\pm$ 3.16} & \textbf{11.49 $\pm$ 0.95} & \textbf{11.75 $\pm$ 0.56} & \textbf{10.37 $\pm$ 0.58} & \textbf{12.09 $\pm$ 1.44} & \textbf{11.14 $\pm$ 1.16} \\
 & T7 &  & 0.57 $\pm$ 0.11 & 1.31 $\pm$ 1.28 & 0.52 $\pm$ 0.14 & 0.41 $\pm$ 0.04 & 0.42 $\pm$ 0.06 & 0.83 $\pm$ 0.59 & 0.40 $\pm$ 0.07 & 0.44 $\pm$ 0.10 \\
\midrule
\rlightgray
 & T1 &  & \textbf{1.59 $\pm$ 0.19} & \textbf{1.93 $\pm$ 0.05} & \textbf{1.91 $\pm$ 0.08} & \textbf{1.83 $\pm$ 0.01} & \textbf{1.88 $\pm$ 0.13} & \textbf{1.85 $\pm$ 0.05} & \textbf{1.99 $\pm$ 0.31} & \textbf{1.67 $\pm$ 0.04} \\
\rlightgray
 & T2 &  & \textbf{22.54 $\pm$ 3.75} & \textbf{15.60 $\pm$ 0.82} & \textbf{20.00 $\pm$ 4.16} & \textbf{17.60 $\pm$ 0.29} & \textbf{21.25 $\pm$ 2.89} & \textbf{19.87 $\pm$ 0.84} & \textbf{18.77 $\pm$ 0.68} & \textbf{16.47 $\pm$ 1.12} \\
\rlightgray
 & T3 &  & \textbf{0.95 $\pm$ 0.30} & 1.36 $\pm$ 0.51 & 1.34 $\pm$ 0.97 & 1.22 $\pm$ 0.66 & 1.41 $\pm$ 0.74 & 1.39 $\pm$ 0.76 & 1.28 $\pm$ 0.70 & 1.17 $\pm$ 0.67 \\
\rlightgray
 & T4 &  & \textbf{0.32 $\pm$ 0.06} & 0.38 $\pm$ 0.11 & \textbf{0.30 $\pm$ 0.01} & \textbf{0.24 $\pm$ 0.03} & \textbf{0.31 $\pm$ 0.05} & \textbf{0.29 $\pm$ 0.04} & \textbf{0.28 $\pm$ 0.01} & \textbf{0.24 $\pm$ 0.00} \\
\rlightgray
 & T5 &  & \textbf{64.73 $\pm$ 15.49} & \textbf{64.23 $\pm$ 7.92} & \textbf{54.06 $\pm$ 5.05} & \textbf{53.67 $\pm$ 2.84} & \textbf{50.14 $\pm$ 0.94} & \textbf{71.29 $\pm$ 7.43} & \textbf{67.00 $\pm$ 18.93} & \textbf{52.35 $\pm$ 4.06} \\
\rlightgray
 & T6 &  & \textbf{13.30 $\pm$ 2.47} & \textbf{13.66 $\pm$ 2.82} & 16.70 $\pm$ 9.04 & 12.20 $\pm$ 2.13 & 22.96 $\pm$ 14.33 & 11.79 $\pm$ 3.73 & 14.49 $\pm$ 5.21 & 13.91 $\pm$ 4.56 \\
\rlightgray
\multirow{-7}{*}{+ \method{}} & T7 & \multirow{-7}{*}{\textbf{442.17 $\pm$ 3.41}} & \textbf{0.38 $\pm$ 0.03} & \textbf{0.58 $\pm$ 0.26} & \textbf{0.41 $\pm$ 0.02} & 0.41 $\pm$ 0.09 & \textbf{0.38 $\pm$ 0.02} & \textbf{0.51 $\pm$ 0.16} & \textbf{0.32 $\pm$ 0.03} & \textbf{0.37 $\pm$ 0.05} \\
\bottomrule
\end{tabular}
}
\end{table}

\clearpage
\subsection{PINN Results}\label{app:pinn-detailed}

Tabs.~\ref{tb:pinn-vanilla}--\ref{tb:pinn-architecture} report the per-PDE mean L2RE on PINNacle for the three baseline families summarized in Sec.~\ref{sec:experiments}, each paired with \method{}.
\begin{table}[!htbp]
\centering
\scriptsize
\setlength{\tabcolsep}{2.5pt}
\caption{Mean L2RE on the \textbf{Vanilla} PINN family (PINN). Bold marks the better Base/+\method{} value (lower is better).}
\label{tb:pinn-vanilla}
\resizebox{\linewidth}{!}{%
\begin{tabular}{lccccccccccccccccccc}
\toprule
\multirow{2}{*}{Method} & \multicolumn{2}{c}{Burgers} & \multicolumn{3}{c}{Poisson} & \multicolumn{4}{c}{Heat} & \multicolumn{3}{c}{NS} & \multicolumn{3}{c}{Wave} & \multicolumn{2}{c}{Chaotic} & \multicolumn{2}{c}{High dim} \\
\cmidrule(lr){2-3} \cmidrule(lr){4-6} \cmidrule(lr){7-10} \cmidrule(lr){11-13} \cmidrule(lr){14-16} \cmidrule(lr){17-18} \cmidrule(lr){19-20}
 & 1d-C & 2d-C & 2d-C & 3d-CG & 2d-MS & 2d-VC & 2d-MS & 2d-CG & 2d-LT & 2d-C & 2d-CG & 2d-LT & 1d-C & 2d-CG & 2d-MS & GS & KS & PNd & HNd \\
\midrule
PINN~\citep{raissi2019physics}          & 1.45E-2 & \textbf{3.24E-1} & 6.94E-1 & 5.60E-1 & 6.30E-1 & 1.01E+0 & 6.21E-2 & 3.64E-2 & 9.99E-1 & 4.70E-2 & 1.19E-1 & 9.96E-1 & 5.88E-1 & 1.84E+0 & 1.34E+0 & 3.19E-1 & 1.01E+0 & 3.04E-3 & \textbf{3.61E-1} \\
\rlightgray
+ \method{}   & \textbf{1.29E-2} & 5.16E-1 & \textbf{6.62E-1} & \textbf{1.94E-1} & \textbf{6.25E-1} & \textbf{9.40E-1} & \textbf{2.83E-2} & \textbf{1.82E-2} & \textbf{9.98E-1} & \textbf{4.15E-2} & \textbf{9.74E-2} & \textbf{9.94E-1} & \textbf{4.86E-1} & \textbf{1.00E+0} & \textbf{8.05E-1} & \textbf{9.32E-2} & \textbf{9.71E-1} & \textbf{7.86E-4} & 5.46E-1 \\
\bottomrule
\end{tabular}
}
\end{table}

\begin{table}[!htbp]
\centering
\scriptsize
\setlength{\tabcolsep}{2.5pt}
\caption{Mean L2RE on the \textbf{Loss Reweighting/Sampling} PINN family (LRA, NTK, RAR). Bold marks the better Base/+\method{} value (lower is better); NaN denotes non-convergence.}
\label{tb:pinn-reweighting}
\resizebox{\linewidth}{!}{%
\begin{tabular}{lccccccccccccccccccc}
\toprule
\multirow{2}{*}{Method} & \multicolumn{2}{c}{Burgers} & \multicolumn{3}{c}{Poisson} & \multicolumn{4}{c}{Heat} & \multicolumn{3}{c}{NS} & \multicolumn{3}{c}{Wave} & \multicolumn{2}{c}{Chaotic} & \multicolumn{2}{c}{High dim} \\
\cmidrule(lr){2-3} \cmidrule(lr){4-6} \cmidrule(lr){7-10} \cmidrule(lr){11-13} \cmidrule(lr){14-16} \cmidrule(lr){17-18} \cmidrule(lr){19-20}
 & 1d-C & 2d-C & 2d-C & 3d-CG & 2d-MS & 2d-VC & 2d-MS & 2d-CG & 2d-LT & 2d-C & 2d-CG & 2d-LT & 1d-C & 2d-CG & 2d-MS & GS & KS & PNd & HNd \\
\midrule
LRA~\citep{wang2021understanding}           & 2.61E-2 & \textbf{2.60E-1} & 1.17E-1 & 1.02E-1 & 7.94E-1 & 2.12E-1 & 8.79E-2 & 1.25E-1 & 9.99E-1 & NaN & 3.32E-1 & 1.00E+0 & 3.61E-1 & 1.48E+0 & 1.02E+0 & \textbf{9.37E-2} & 9.57E-1 & 4.58E-4 & 3.94E-1 \\
\rlightgray
+ \method{}   & \textbf{2.04E-2} & 4.64E-1 & \textbf{2.62E-2} & \textbf{8.47E-2} & \textbf{7.11E-1} & \textbf{2.01E-1} & \textbf{3.93E-2} & \textbf{1.08E-1} & 9.99E-1 & NaN & \textbf{3.01E-1} & 1.00E+0 & \textbf{1.25E-1} & \textbf{1.01E+0} & \textbf{9.96E-1} & 1.24E-1 & \textbf{9.52E-1} & 4.58E-4 & \textbf{2.85E-1} \\
\midrule
NTK~\citep{wang2022ntk}           & 1.84E-2 & \textbf{2.75E-1} & 1.23E-2 & 9.47E-1 & 7.48E-1 & 2.14E-1 & 4.40E-2 & 1.16E-1 & 1.00E+0 & 1.98E-1 & 2.93E-1 & 9.99E-1 & 9.79E-2 & 2.16E+0 & 1.04E+0 & 2.16E-1 & 9.64E-1 & \textbf{4.64E-3} & 3.97E-1 \\
\rlightgray
+ \method{}   & \textbf{1.48E-2} & 4.64E-1 & \textbf{5.60E-3} & \textbf{9.41E-1} & \textbf{7.18E-1} & \textbf{2.00E-1} & \textbf{2.41E-2} & \textbf{8.71E-2} & \textbf{9.98E-1} & \textbf{1.76E-1} & \textbf{2.80E-1} & 9.99E-1 & \textbf{8.51E-2} & \textbf{9.90E-1} & \textbf{7.91E-1} & \textbf{1.85E-1} & \textbf{9.47E-1} & 4.74E-3 & \textbf{3.92E-1} \\
\midrule
RAR~\citep{wu2023comprehensive}           & 3.32E-2 & \textbf{3.45E-1} & 6.99E-1 & 5.76E-1 & \textbf{6.44E-1} & 9.66E-1 & 7.49E-2 & 2.72E-2 & 9.99E-1 & 4.69E-1 & 3.34E-1 & 1.00E+0 & 5.39E-1 & 1.15E+0 & 1.35E+0 & \textbf{9.46E-2} & 1.01E+0 & 3.59E-3 & \textbf{3.57E-1} \\
\rlightgray
+ \method{}   & \textbf{2.43E-2} & 5.21E-1 & \textbf{6.90E-1} & \textbf{5.12E-1} & 6.51E-1 & \textbf{9.63E-1} & \textbf{3.18E-2} & \textbf{1.84E-2} & 9.99E-1 & \textbf{4.08E-2} & \textbf{1.02E-1} & \textbf{9.94E-1} & \textbf{4.96E-1} & \textbf{1.04E+0} & \textbf{9.09E-1} & 1.93E-1 & 1.01E+0 & \textbf{3.21E-3} & 3.81E-1 \\
\bottomrule
\end{tabular}
}
\end{table}

\begin{table}[!htbp]
\centering
\scriptsize
\setlength{\tabcolsep}{2.5pt}
\caption{Mean L2RE on the \textbf{Architecture} PINN family (LAAF, GAAF). Bold marks the better Base/+\method{} value (lower is better).}
\label{tb:pinn-architecture}
\resizebox{\linewidth}{!}{%
\begin{tabular}{lccccccccccccccccccc}
\toprule
\multirow{2}{*}{Method} & \multicolumn{2}{c}{Burgers} & \multicolumn{3}{c}{Poisson} & \multicolumn{4}{c}{Heat} & \multicolumn{3}{c}{NS} & \multicolumn{3}{c}{Wave} & \multicolumn{2}{c}{Chaotic} & \multicolumn{2}{c}{High dim} \\
\cmidrule(lr){2-3} \cmidrule(lr){4-6} \cmidrule(lr){7-10} \cmidrule(lr){11-13} \cmidrule(lr){14-16} \cmidrule(lr){17-18} \cmidrule(lr){19-20}
 & 1d-C & 2d-C & 2d-C & 3d-CG & 2d-MS & 2d-VC & 2d-MS & 2d-CG & 2d-LT & 2d-C & 2d-CG & 2d-LT & 1d-C & 2d-CG & 2d-MS & GS & KS & PNd & HNd \\
\midrule
LAAF~\citep{jagtap2020adaptive}          & 1.43E-2 & \textbf{2.77E-1} & 7.68E-1 & 5.79E-1 & 5.93E-1 & \textbf{6.42E-1} & \textbf{7.40E-2} & \textbf{2.39E-2} & 9.99E-1 & 3.60E-2 & 8.24E-2 & 9.98E-1 & 4.54E-1 & 8.19E-1 & 1.06E+0 & \textbf{9.47E-2} & 1.01E+0 & 4.14E-3 & 5.22E-1 \\
\rlightgray
+ \method{}   & \textbf{1.29E-2} & 4.94E-1 & \textbf{6.35E-1} & \textbf{2.04E-1} & \textbf{4.69E-1} & 8.66E-1 & 1.01E-1 & 2.45E-2 & \textbf{9.98E-1} & \textbf{2.13E-2} & \textbf{6.63E-2} & \textbf{9.97E-1} & \textbf{4.20E-1} & \textbf{7.62E-1} & \textbf{1.03E+0} & 9.48E-2 & \textbf{9.98E-1} & \textbf{4.10E-3} & \textbf{4.96E-1} \\
\midrule
GAAF          & 5.20E-2 & \textbf{2.95E-1} & \textbf{6.04E-1} & 5.02E-1 & \textbf{9.31E-1} & \textbf{8.49E-1} & 9.85E-1 & 4.61E-1 & \textbf{9.99E-1} & 3.79E-2 & 1.74E-1 & 9.99E-1 & 6.77E-1 & 7.94E-1 & 1.06E+0 & 9.46E-2 & \textbf{1.00E+0} & 7.75E-3 & 5.21E-1 \\
\rlightgray
+ \method{}   & \textbf{2.01E-2} & 5.08E-1 & 6.06E-1 & \textbf{2.30E-1} & 9.79E-1 & 1.06E+0 & \textbf{1.26E-1} & \textbf{4.60E-1} & 1.00E+0 & \textbf{3.62E-2} & \textbf{1.67E-1} & 9.99E-1 & \textbf{6.59E-1} & \textbf{7.24E-1} & \textbf{1.03E+0} & \textbf{9.45E-2} & 1.01E+0 & \textbf{7.70E-3} & \textbf{4.85E-1} \\
\bottomrule
\end{tabular}
}
\end{table}

\clearpage
\subsection{Recent MTL Solvers}\label{app:recent-solvers}

The MTL solvers in Sec.~\ref{sec:experiments} were proposed between 2018 and 2021. To verify that \method{} also composes with more recent solvers, we wrap Nash-MTL~\citep{navon2022multi}, FairGrad~\citep{ban2024fair}, and ExcessMTL~\citep{he2024robust} on $\mathtt{SARCOS}$ under the same seven-task setup, Adam configuration, schedule, and three seeds as Tab.~\ref{tab:sarcos-mtl}. Nash-MTL and FairGrad are gradient-balancing solvers that determine $\boldsymbol{\lambda}^{(t)}$ from the Gram matrix $\mathbf{K}^{(t)}$ (App.~\ref{app:moo-gb}); ExcessMTL is a loss-balancing solver that sets $\boldsymbol{\lambda}^{(t)}$ from per-task excess risks (App.~\ref{app:moo-lb}). All three emit the linear form of Eq.~\eqref{eq:appendix-linear-form} with $\boldsymbol{\theta}$-detached coefficients, so \method{} applies without modification. Tab.~\ref{tab:sarcos-recent-solvers} reports the results: \method{} lowers the average error for all three solvers and improves $20$ of the $21$ task-wise means. The single exception is $t_1$ under Nash-MTL, where the two values differ by $0.02$ and lie within one standard deviation.

\begin{table}[!htbp]
\centering
\scriptsize
\setlength{\tabcolsep}{3pt}
\caption{$\mathtt{SARCOS}$ MTL results for recent solvers (mean $\pm$ std over three seeds). Errors are scaled by $100$. Bold marks the better value within each \mbox{Base/+\method{}} block.}
\label{tab:sarcos-recent-solvers}
\resizebox{\linewidth}{!}{%
\begin{tabular}{lrrrrrrrr}
\toprule
Method & Avg. err. $\downarrow$ & $t_1$ $\downarrow$ & $t_2$ $\downarrow$ & $t_3$ $\downarrow$ & $t_4$ $\downarrow$ & $t_5$ $\downarrow$ & $t_6$ $\downarrow$ & $t_7$ $\downarrow$ \\
\midrule
Nash-MTL~\citep{navon2022multi} & 11.39 $\pm$ 0.54 & \textbf{0.83 $\pm$ 0.02} & 14.46 $\pm$ 1.24 & 0.40 $\pm$ 0.02 & 0.09 $\pm$ 0.01 & 56.62 $\pm$ 2.49 & 7.14 $\pm$ 0.10 & 0.16 $\pm$ 0.02 \\
\rlightgray
+ \method{} & \textbf{10.99 $\pm$ 0.30} & 0.85 $\pm$ 0.02 & \textbf{13.38 $\pm$ 0.19} & \textbf{0.37 $\pm$ 0.01} & \textbf{0.08 $\pm$ 0.01} & \textbf{55.24 $\pm$ 1.90} & \textbf{6.87 $\pm$ 0.05} & \textbf{0.14 $\pm$ 0.00} \\
\midrule
FairGrad~\citep{ban2024fair} & 11.42 $\pm$ 0.16 & 0.87 $\pm$ 0.05 & 14.08 $\pm$ 0.70 & 0.39 $\pm$ 0.04 & 0.10 $\pm$ 0.02 & 57.07 $\pm$ 1.94 & 7.22 $\pm$ 0.18 & 0.17 $\pm$ 0.02 \\
\rlightgray
+ \method{} & \textbf{11.11 $\pm$ 0.29} & \textbf{0.82 $\pm$ 0.03} & \textbf{13.36 $\pm$ 0.27} & \textbf{0.38 $\pm$ 0.03} & \textbf{0.07 $\pm$ 0.00} & \textbf{56.09 $\pm$ 3.26} & \textbf{6.92 $\pm$ 0.19} & \textbf{0.14 $\pm$ 0.00} \\
\midrule
ExcessMTL~\citep{he2024robust} & 11.09 $\pm$ 0.61 & 1.75 $\pm$ 0.11 & 15.60 $\pm$ 0.54 & 0.70 $\pm$ 0.05 & 0.22 $\pm$ 0.02 & 50.15 $\pm$ 4.49 & 8.91 $\pm$ 0.32 & 0.33 $\pm$ 0.01 \\
\rlightgray
+ \method{} & \textbf{10.43 $\pm$ 0.46} & \textbf{1.65 $\pm$ 0.10} & \textbf{14.71 $\pm$ 0.37} & \textbf{0.67 $\pm$ 0.02} & \textbf{0.20 $\pm$ 0.01} & \textbf{47.12 $\pm$ 3.44} & \textbf{8.36 $\pm$ 0.26} & \textbf{0.31 $\pm$ 0.02} \\
\bottomrule
\end{tabular}
}
\end{table}
\subsection{Curvature-Aware Optimizers}\label{app:curvature-baselines}

To separate the effect of the preference-conditioned metric from that of curvature information in general, we compare \method{} against curvature-aware optimizers on $\mathtt{SARCOS}$ under LS, using the setup of Tab.~\ref{tab:sarcos-mtl} and three seeds. FAdam~\citep{hwang2024fadam} preconditions by a diagonal empirical Fisher estimated from the update direction, Shampoo~\citep{gupta2018shampoo} maintains full preconditioners along each tensor axis, and AdaHessian~\citep{yao2021adahessian} estimates the Hessian diagonal by Hutchinson sampling. All optimizers use the Adam hyperparameters of the baseline without optimizer-specific retuning; the Adam and \method{} rows coincide with the LS block of Tab.~\ref{tab:sarcos-mtl}. Tab.~\ref{tab:sarcos-curvature} reports the results. Adam and FAdam perform similarly: FAdam estimates its Fisher from the scalarized direction $\mathbf{d}^{(t)}$ and therefore inherits the weight marginalization of Prop.~\ref{prop:adam-weighting-mismatch}. Shampoo and AdaHessian degrade substantially under the shared hyperparameters, indicating that they require optimizer-specific calibration to transfer to this setting. \method{} reuses the Adam configuration and reduces the average error from $10.99$ to $10.00$, which attributes the gain to the preference-conditioned cross-objective curvature rather than to curvature information alone.

\begin{table}[!htbp]
\centering
\scriptsize
\setlength{\tabcolsep}{3pt}
\caption{$\mathtt{SARCOS}$ MTL results under LS for curvature-aware optimizers (mean over three seeds). Errors are scaled by $100$. Bold marks the best value in each column.}
\label{tab:sarcos-curvature}
\resizebox{\linewidth}{!}{%
\begin{tabular}{lrrrrrrrr}
\toprule
Optimizer & Avg. err. $\downarrow$ & $t_1$ $\downarrow$ & $t_2$ $\downarrow$ & $t_3$ $\downarrow$ & $t_4$ $\downarrow$ & $t_5$ $\downarrow$ & $t_6$ $\downarrow$ & $t_7$ $\downarrow$ \\
\midrule
Adam & 10.99 & 1.63 & 14.87 & 0.65 & \textbf{0.19} & 50.32 & 8.96 & \textbf{0.30} \\
FAdam~\citep{hwang2024fadam} & 11.40 & 1.65 & 15.16 & 0.70 & 0.20 & 52.38 & 9.44 & 0.31 \\
Shampoo~\citep{gupta2018shampoo} & 41.33 & 3.61 & 73.91 & 1.47 & 0.28 & 183.45 & 26.20 & 0.39 \\
AdaHessian~\citep{yao2021adahessian} & 99.01 & 10.27 & 238.84 & 5.49 & 4.00 & 366.31 & 64.31 & 3.84 \\
\rlightgray
\method{} & \textbf{10.00} & \textbf{1.46} & \textbf{13.55} & \textbf{0.61} & \textbf{0.19} & \textbf{45.43} & \textbf{8.47} & 0.31 \\
\bottomrule
\end{tabular}
}
\end{table}

\clearpage
\section{Computational Cost}\label{app:cost}

\paragraph{Per-step overhead.}
We profile \method{} on $\mathtt{SARCOS}$ with an NVIDIA A6000 GPU, using five paired profiles per method of the same training run with and without \method{}. Tab.~\ref{tab:overhead} reports steady-state wall-clock time per step and peak allocated memory for LS (loss balancing) and CAGrad (gradient balancing). The added wall-clock time comes from the two sampled per-objective backward passes of Sec.~\ref{sec:practical} and is similar in absolute terms for both solvers (about $7$\,ms); the relative increase is smaller for CAGrad because its base step already computes per-objective gradients and is therefore more expensive. The added memory holds the $C(C+1)/2$ cross-objective Fisher EMAs of Eq.~\eqref{eq:fisher-ema}, which scale as $O(C^2)$ in the number of objectives and linearly in the shared-parameter dimension. Only the number of sampled per-objective gradient evaluations per step is independent of $C$; total runtime and memory are not.

\begin{table}[!htbp]
\centering
\small
\caption{Per-step cost on $\mathtt{SARCOS}$ ($C=7$; NVIDIA A6000; mean over five paired profiles). Wall-clock is per optimizer step; peak memory is peak allocated GPU memory.}
\label{tab:overhead}
\begin{tabular}{lrrrrr}
\toprule
& \multicolumn{3}{c}{Wall-clock (ms/step)} & \multicolumn{2}{c}{Peak memory (MiB)} \\
\cmidrule(lr){2-4} \cmidrule(lr){5-6}
Solver & Base & +\method{} & $\Delta$ & Base & +\method{} \\
\midrule
LS     & 36.09 & 43.00 & $+19.3\%$ & 20.98 & 39.91 \\
CAGrad & 56.59 & 63.33 & $+12.1\%$ & 31.60 & 46.78 \\
\bottomrule
\end{tabular}
\end{table}

\paragraph{Matched-quality cost.}
Per-step overhead does not capture the total computation required to reach a given solution quality. Tab.~\ref{tab:matched-quality} compares CAGrad under Adam and under \method{} on $\mathtt{SARCOS}$ at a matched validation error of $0.479 \pm 0.052$ (mean $\pm$ std over three seeds), reporting the number of steps and the cumulative FLOPs needed to reach it. \method{} reaches the target in $33\%$ fewer steps and with $20\%$ fewer cumulative FLOPs, so its faster convergence offsets the per-step overhead in this comparison.

\begin{table}[!htbp]
\centering
\small
\caption{Steps and cumulative compute to reach a matched validation error of $0.479 \pm 0.052$ for CAGrad on $\mathtt{SARCOS}$ (mean $\pm$ std over three seeds).}
\label{tab:matched-quality}
\begin{tabular}{lrr}
\toprule
Method & Steps to target $\downarrow$ & Cumulative TFLOPs $\downarrow$ \\
\midrule
CAGrad + Adam      & $5{,}965 \pm 1{,}041$ & $14.10 \pm 2.46$ \\
\rlightgray
CAGrad + \method{} & $\mathbf{4{,}009 \pm 1{,}084}$ & $\mathbf{11.31 \pm 3.04}$ \\
\bottomrule
\end{tabular}
\end{table}

\end{document}